\documentclass[11pt]{article}
\usepackage{latexsym,amssymb,amsmath} % for \Box, \mathbb, split, etc.
\usepackage{path}
\usepackage{url}
\usepackage{verbatim}

\usepackage{natbib}
\usepackage{amsfonts}
\usepackage{amsfonts}
\usepackage{array}
\usepackage{mathrsfs}
\usepackage{amsfonts}
\usepackage{amsmath}
\usepackage{amssymb}
\usepackage{pifont}
\usepackage{amsthm}
\usepackage{bm}
\usepackage{appendix}
\usepackage{multirow}
\usepackage{pifont}

\usepackage{color}
\usepackage[table]{xcolor}
\usepackage{thmtools}
\usepackage{thm-restate}
\usepackage{enumitem}
\usepackage{xcolor}         % colors
\usepackage{colortbl}
\usepackage{graphicx}
\usepackage{subfigure}

\usepackage[colorlinks=true,citecolor=blue]{hyperref}
\setcitestyle{open={(},close={)}}

\usepackage[lined,boxed,ruled, commentsnumbered]{algorithm2e}

\newtheorem{lemma}{Lemma}
\newtheorem{theorem}{Theorem}

\newtheorem{definition}{Definition}
\newtheorem{remark}{Remark}

\usepackage{multirow}
\usepackage{pifont}
\usepackage{colortbl}
\usepackage{graphicx}
\usepackage{subfigure}

\newcommand{\argmin}{\mathop{\arg\min}}

\newcommand{\tabincell}[2]{\begin{tabular}{@{}#1@{}}#2\end{tabular}}

\newcommand{\FedMSPP}{\texttt{FedMSPP}}
\newcommand{\FedProx}{\texttt{FedProx}}
\newcommand{\FedAvg}{\texttt{FedAvg}}
\newcommand{\SCAFFOLD}{\texttt{SCAFFOLD}}
\newcommand{\FedPD}{\texttt{FedPD}}
\newcommand{\FCO}{\texttt{FCO}}

\newcommand{\STEM}{\texttt{STEM}}
\newcommand{\VRLSGD}{\texttt{VRL-SGD}}
\newcommand{\FedHT}{\texttt{Fed-HT}}
\newcommand{\erm}{\texttt{erm}}
\newcommand{\prox}{\text{prox}}
\newcommand{\xmark}{\textcolor[rgb]{1.00,0.00,0.00}{\ding{55}}}%
\newcommand{\cmark}{\textcolor[rgb]{0.00,0.700,0.00}{\ding{51}}}


\numberwithin{equation}{section}
\title{On Convergence of FedProx: Local Dissimilarity Invariant Bounds, Non-smoothness and Beyond}
\author{\vspace{0.4in}\\
  \textbf{Xiao-Tong Yuan} \ and \ \textbf{Ping Li} \\\\
  Cognitive Computing Lab \\
   Baidu Research \\
  No. 10 Xibeiwang East Road, Beijing 100193, China \\
  10900 NE 8th St, Bellevue, Washington 98004, USA\\
  E-mail: \texttt{\{xtyuan1980, pingli98\}@gmail.com}
  }
\date{}

\begin{document}

\maketitle

\begin{abstract}\vspace{0.3in}

\noindent The \FedProx~algorithm is a simple yet powerful distributed proximal point optimization method widely used for federated learning (FL) over heterogeneous data. Despite its popularity and remarkable success witnessed in practice, the theoretical understanding of FedProx is largely underinvestigated: the appealing convergence behavior of \FedProx~is so far characterized under certain non-standard and unrealistic dissimilarity assumptions of local functions, and the results are limited to smooth optimization problems. In order to remedy these deficiencies, we develop a novel local dissimilarity invariant convergence theory for \FedProx~and its minibatch stochastic extension through the lens of algorithmic stability. As a result, we contribute to derive several new and deeper insights into \FedProx~for non-convex federated optimization including: 1) convergence guarantees independent on local dissimilarity type conditions; 2) convergence guarantees for non-smooth FL problems; and 3) linear speedup with respect to size of minibatch and number of sampled devices. Our theory for the first time reveals that local dissimilarity and smoothness are not must-have for \FedProx~to get favorable complexity bounds. Preliminary experimental results on a series of benchmark FL datasets are reported to demonstrate the benefit of minibatching for improving the sample efficiency of \FedProx.
\end{abstract}

\subparagraph{Key words.} Federated learning, FedProx, Minibatch stochastic proximal point methods, Uniform stability, Non-convex optimization, Non-smooth optimization

\newpage

\section{Introduction}

Federated Learning (FL) has recently emerged as a promising paradigm for communication-efficient distributed learning on remote devices, such as smartphones, internet of things, or agents~\citep{konevcny2016federated,yang2019federated}. The goal of FL is to collaboratively train a shared model that works favorably for all the local data but without requiring the learners to transmit raw data across the network. The principle of optimizing a global model while keeping data localized can be beneficial for both computational efficiency and data privacy~\citep{bhowmick2018protection}. While resembling the classic distributed learning regimes, there are two most distinct features associated with FL: 1) large statistical heterogeneity of local data mainly due to the non-iid manner of data generalization and collection across the devices~\citep{hard2020training}; and 2) partial participants of devices in the network mainly due to the massive number of devices. These fundamental challenges make FL highly demanding to tackle, both in terms of optimization algorithm design and in terms of theoretical understanding of convergence behavior~\citep{li2020federated}.

FL is most conventionally formulated as the following problem of global population risk minimization averaged over a set of $M$ devices:
\begin{equation}\label{equat:problem}
\min_{w \in \mathbb{R}^p} \bar R(w):= \frac{1}{M}\sum_{m=1}^M \left\{R^{(m)}(w):= \mathbb{E}_{Z^{(m)} \sim \mathcal{D}^{(m)}} [\ell^{(m)}(w; Z^{(m)})]\right\},
\end{equation}
where $R^{(m)}$ is the local population risk on device $m$, $\ell^{(m)}: \mathbb{R}^p \times \mathcal{Z}^{(m)} \mapsto R^+$ is a non-negative loss function whose value $\ell(w;Z^{(m)})$ measures the loss over a random data point $Z^{(m)}\in \mathcal{Z}^{(m)}$ with parameter $w$, $\mathcal{D}^{(m)}$ represents an underlying random data distribution over $\mathcal{Z}^{(m)}$. Since the data distribution is typically unknown, the following empirical risk minimization (ERM) version of~\eqref{equat:problem} is often considered alternatively:
\begin{equation}\label{equat:problem_erm}
\min_{w \in \mathbb{R}^p} \bar R_{\erm}(w):= \frac{1}{M}\sum_{m=1}^M \left\{R_{\erm}^{(m)}(w):= \frac{1}{N_m}\sum_{i=1}^{N_m} \ell^{(m)}(w; z^{(m)}_{i})\right\},
\end{equation}
where $R_{\erm}^{(m)}$ is the local empirical risk over the training sample $D^{(m)}=\{z^{(m)}_{i}\}_{i=1}^{N_m}$ on device $m$. The sample size $N_m$ may vary significantly across devices, which can be regarded as another source of data heterogeneity. Federated optimization algorithms for solving~\eqref{equat:problem} or \eqref{equat:problem_erm} have attracted significant research interest from both academia and industry, with a rich body of efficient solutions developed that can flexibly adapt to the communication-computation tradeoffs and data/system heterogeneity. Several popularly used FL algorithms for this setting include \FedAvg~\citep{mcmahan2017communication}, \FedProx~\citep{li2020federatedprox}, \SCAFFOLD~\citep{karimireddy2020scaffold}, and \FedPD~\citep{zhang2020fedpd}, to name a few. A consensus among these methods on communication-efficient implementation is trying to extensively update the local models (e.g., with plenty epochs of local optimization) over subsets of devices so as to quickly find an optimal global model using a minimal number of inter-device communication rounds for model aggregation.

In this paper, we revisit the \FedProx~algorithm which is one of the most prominent frameworks for heterogeneous federated optimization. Reasons for the interests of \FedProx~include implementation simplicity, low communication cost, promise in dealing with data heterogeneity and tolerance to partial participants of devices~\citep{li2020federatedprox}. We analyze its convergence behavior, expose problems, and propose alternatives more suitable for scaling up and generalization. We contribute to derive some new and deeper theoretical insights into the algorithm from a novel perspective of algorithmic stability theory.

\subsection{Review of \FedProx}

For solving FL problems in the presence of data heterogeneity, methods such as \FedAvg~based on local stochastic gradient descent (SGD) can fail to converge in practice when the selected devices perform too many local updates~\citep{li2020federatedprox}. To mitigate this issue, \FedProx~\citep{li2020federatedprox} was recently proposed for solving the empirical FL problem~\eqref{equat:problem_erm} using the (inexact) proximal point update for local optimization. The benefits of \FedProx~include: 1) it provides more stable local updates by explicitly enforcing the local optimization in the vicinity of the global model to date; 2) the method comes with convergence guarantees for both convex and non-convex functions, even under partial participation and very dissimilar amounts of local updates~\citep{li2020federated}. More specifically, at each time instance $t$, \FedProx~uniformly randomly selects a subset $I_t\subseteq [M]$ of devices and introduces for each device $\xi \in I_t$ the following proximal point ERM sub-problem for local update around the previous global model $w_{t-1}$:
\begin{equation}\label{equat:fedprox_iteration}
w^{(\xi)}_t \approx \argmin_{w\in \mathbb{R}^p} \left\{Q^{(\xi)}_{\erm}(w; w_{t-1}):=R^{(\xi)}_{\erm}(w) + \frac{1}{2\eta_t}\|w - w_{t-1}\|^2\right\},
\end{equation}
where $\eta_t>0$ is the learning rate that controls the impact of the proximal term. Then the global model is updated by uniformly aggregating those local updates from $I_t$ as
\[
w_t = \frac{1}{|I_t|}\sum_{\xi\in I_t} w^{(\xi)}_t.
\]
In the extreme case of allowing $\eta_t \rightarrow +\infty$ in~\eqref{equat:fedprox_iteration}, \FedProx~reduces to the regime of \FedAvg~if using SGD for local optimization. Since its inception, \FedProx~and its variants have received significant interests in research~\citep{li2019feddane,nguyen2020fast,pathak2020fedsplit} and become an algorithm of choice in application areas such as automatous driving~\citep{donevski2021addressing} and computer vision~\citep{he2021fedcv}. Theoretically, \FedProx~comes with convergence guarantees under the following bounded \emph{local gradient dissimilarity} assumption that captures the statistical heterogeneity of local objectives across the network:
\begin{definition}[$(B,H)$-LGD]\label{def:local_dissimilarity}
We say the local functions $R^{(m)}$ have $(B,H)$-local gradient dissimilarity (LGD) if the following holds for all $w\in \mathbb{R}^p$:
\[
\frac{1}{M}\sum_{m=1}^M \|\nabla R^{(m)}(w)\|^2 \le B^2 \|\nabla \bar R(w)\|^2 + H^2.
 \]
 The definition naturally extends to the local empirical risks $\{R^{(m)}_{\erm}\}_{m=1}^M$.
\end{definition}
Specially in the homogenous setting where $R^{(m)}\equiv\bar R$, $\forall m\in [M]$, we have $B=1$ and $H=0$. Under $(B,0)$-LGD and some regularization condition on the modulus $B$, it was shown that \FedProx~for non-convex problems requires $T=\mathcal{O}\left(\frac{1}{\epsilon}\right)$ rounds of inter-device communication to reach an $\epsilon$-stationary solution, i.e., $\frac{1}{T}\sum_{t=1}^T \|\nabla \bar R_{\erm}( w_t)\|^2\le\epsilon$~\citep{li2020federatedprox}. Similar guarantees have also been established for a variant of \FedProx~with non-uniform model aggregation schemes~\citep{nguyen2020fast}.

\newpage

\textbf{Open issues and motivation.} In spite of the remarkable success achieved by \FedProx~and its variants, there are still a number of important theoretical issues regarding the unrealistic assumptions, restrictive problem regimes and expensive local oracle cost that remain open for exploration, as specified below.
\begin{itemize}%[leftmargin=*]
  \item \textbf{Local dissimilarity}. The appealing convergence behavior of \FedProx~is so far characterized under a key but non-standard $(B,H)$-LGD (cf. Definition~\ref{def:local_dissimilarity}) condition with $B>0$ and $H=0$. Such a condition is obviously unrealistic in practice: it essentially requires the local objectives share the same stationary point as the global objective since $\|\nabla \bar R_{\erm}(w)\|=0$ implies $\|\nabla R^{(m)}_{\erm}(w)\|=0$ for all $m\in [M]$. However, if the optima of $R^{(m)}_{\erm}$ are exactly (or even approximately) the same, there would be little point in distributing data across devices for federated learning. \emph{It is thus desirable to understand the convergence behavior of \FedProx~for heterogeneous FL without imposing stringent local dissimilarity conditions like $(B,0)$-LGD with $B>0$.}
  \item \textbf{Non-smooth optimization}. The existing convergence guarantees of \FedProx~are only available for FL with smooth losses. More often than not, however, FL applications involve non-smooth objectives due to the popularity of non-smooth losses (e.g., hinge loss and absolute loss) in machine learning, and training deep neural networks with non-smooth activation like ReLU. \emph{Therefore, it is desirable to understand the convergence behavior of \FedProx~in non-smooth problem regimes.}
  \item \textbf{Local oracle complexity}. Unlike the (stochastic) first-order oracles such as SGD used by \FedAvg, the proximal point oracle~\eqref{equat:fedprox_iteration} for local update is by itself a full-batch ERM problem which tends to be expensive to solve even approximately per-iteration. Plus, due to the potentially imbalanced data distribution over devices, the computational overload of the proximal point oracle could vary significantly across the network. \emph{Therefore, it is important to investigate whether using the stochastic approximation to the proximal point oracle~\eqref{equat:fedprox_iteration} can provably improve the computational efficiency of \FedProx.}
\end{itemize}
Last but not least, existing convergence analysis of \FedProx~mainly focuses on the empirical FL problem~\eqref{equat:problem_erm}. The optimality in terms of the population FL problem~\eqref{equat:problem} is not yet clear for \FedProx. The primary goal of this work is to remedy these theoretical issues simultaneously, so as to lay a more solid theoretical foundation for the popularly applied \FedProx~algorithm.

\subsection{Our Contributions}
\label{ssect:overview}
In this paper, we make progress towards understanding the convergence behavior of \FedProx~for non-convex heterogenous FL under weaker and more realistic conditions. The main results are a set of local dissimilarity invariant bounds for smooth or non-smooth problems.

\textbf{Main results for the vanilla \FedProx.} As a starting point to address the restrictiveness of local dissimilarity assumption, we provide a novel convergence analysis for the vanilla \FedProx~algorithm independent of local dissimilarity type conditions. For smooth and non-convex optimization problems, our result in Theorem~\ref{thrm:fedprox_main_smooth} shows that the rate of convergence to a stationary point is upper bounded by
\begin{equation}\label{equat:main_fedprox_smooth}
\frac{1}{T}\sum_{t=0}^{T-1} \mathbb{E}\left[\left\|\nabla \bar R_{\erm}(w_{t})\right\|^2\right] \lesssim \max\left\{\frac{1}{T^{2/3}}, \frac{1}{\sqrt{TI}}\right\},
\end{equation}
where $I$ is the number devices randomly selected for local update at each iteration. If all the devices participate in the local updates for every round, i.e. $I_t = [M]$, the rate of convergence can be improved to $\mathcal{O}(\frac{1}{T^{2/3}})$. For $T<I^3$, the rate in~\eqref{equat:main_fedprox_smooth} is dominated by $\mathcal{O}(\frac{1}{T^{2/3}})$ which gives the communication complexity $\frac{1}{\epsilon^{3/2}}$ to achieve an $\epsilon$-stationary solution. On the other hand when $T\ge I^3$, the rate is dominated by $\mathcal{O}(\frac{1}{\sqrt{TI}})$ which gives the communication complexity $\frac{1}{I\epsilon^2}$. Compared to the already known $\mathcal{O}(\frac{1}{\epsilon})$ complexity bound of \FedProx~under the unrealistic $(B,0)$-LGD condition~\citep{li2020federatedprox}, our rate in~\eqref{equat:main_fedprox_smooth} is slower but it holds without needing to impose stringent regularity conditions on the dissimilarity of local functions, and it reveals the effect of device sampling for accelerating convergence. Further for \emph{non-smooth} and non-convex problems, we establish in Theorem~\ref{thrm:fedprox_main_nonsmooth} the following rate of convergence
\begin{equation}\label{equat:main_fedprox_nonsmooth}
\frac{1}{T}\sum_{t=0}^{T-1} \mathbb{E}\left[\left\|\nabla \bar R_{\erm}(w_{t})\right\|^2\right] \lesssim  \frac{1}{\sqrt{T}},
\end{equation}
which is invariant to the number of selected devices in each round. In the case of $I=\mathcal{O}(1)$, the bounds in~\eqref{equat:main_fedprox_smooth} and~\eqref{equat:main_fedprox_nonsmooth} are comparable, which indicates that smoothness is not must-have for \FedProx~to get sharper convergence bound especially with low participation ratio. On the other end when $I=\mathcal{O}(M)$, the bound~\eqref{equat:main_fedprox_nonsmooth} for non-smooth problems is slower than the bound~\eqref{equat:main_fedprox_smooth} for smooth functions in large-scale networks.

\textbf{Main results for minibatch stochastic \FedProx.} Then as the chief contribution of the present work, we propose a minibatch stochastic extension of \FedProx~along with its population optimization performance analysis from a novel perspective of algorithmic stability theory. Inspired by the recent success of minibatch stochastic proximal point methods (MSPP)~\citep{li2014efficient,wang2017memory,asi2020minibatch,deng2021minibatch}, we propose to implement \FedProx~using MSPP as the local update oracle. The resulting method, which is referred to as \FedMSPP, is expected to attain improved trade-off between computation, communication and memory efficiency for large-scale FL. In the case of imbalanced data distribution, minibatching is also beneficial for making the local computation more balanced across the devices. Based on some extended uniform stability arguments for gradients, we show in Theorem~\ref{thrm:fedmspp_main_smooth} the following local dissimilarity invariant rate of convergence for \FedMSPP~in terms of population optimality:
\begin{equation}\label{equat:main_fedmspp_smooth}
\frac{1}{T}\sum_{t=0}^{T-1} \mathbb{E} \left[ \left\|\nabla \bar R(w_{t})\right\|^2 \right] \lesssim \max\left\{\frac{1}{T^{2/3}}, \frac{1}{\sqrt{TbI}}\right\},
\end{equation}
where $b$ is the minibatch size of local update. For empirical FL, identical bound holds under sampling according to empirical distribution. For $T<(bI)^3$, the rate in~\eqref{equat:main_fedmspp_smooth} is dominated by $\mathcal{O}(\frac{1}{T^{2/3}})$ which gives the communication complexity $\frac{1}{\epsilon^{3/2}}$, and it matches that of the vanilla \FedProx. For sufficiently large $T\ge (bI)^3$, the rate is dominated by $\mathcal{O}(\frac{1}{\sqrt{TbI}})$ which gives the communication complexity $\frac{1}{bI\epsilon^2}$. %and IFO complexity $\frac{1}{\epsilon^2}$.
This shows that local minibatching and device sampling are both beneficial for linearly speeding up communication. Further, when applied to non-smooth problems, we can similarly show that \FedMSPP~converges at the rate of
\[
\frac{1}{T}\sum_{t=0}^{T-1} \mathbb{E}\left[\left\|\nabla \bar R(w_{t})\right\|^2\right] \lesssim  \frac{1}{\sqrt{T}},
\]
which is comparable to that of~\eqref{equat:main_fedmspp_smooth} when $b=\mathcal{O}(1)$ and $I=\mathcal{O}(1)$, but without showing the effect of linear speedup with respect to $b$ and $I$.
%Also, the bound is derived without assuming the bounded gradient variance condition as conventionally used in previous analysis for stochastic FL methods~\cite{karimireddy2020scaffold,li2019convergence,khanduri2021stem}.
%If local dissimilarity holds additionally, we further show that the rate of convergence scales as
%\[
%\frac{1}{T}\sum_{t=0}^{T-1} \mathbb{E} \left[ \|\nabla \bar R(\bar w_{t})\|^2 \right] \lesssim \max\left\{\frac{1}{T}, \frac{1}{T^{2/3}b^{1/3}}, \frac{1}{\sqrt{TbI}}\right\}.
%\]
%This improves the bound in~\eqref{inequat:main_bound_indpendent} when $T$ is relatively small, say $T\le M^3b$.

\begin{table}[t]
%\renewcommand{\arraystretch}{1.0}
\centering
\begin{tabular}{|c c c c c c|}
\hline
Method & Work &  Commun. Complex. & LD Independ. & NS & PP \\
\hline
 \cellcolor[gray]{0.77}\FedProx & \citet{li2020federatedprox} & $\mathcal{O}\left(\frac{1}{\epsilon}\right)$ & \xmark & \xmark & \cmark \\
\cellcolor[gray]{0.77}& \cellcolor[gray]{0.77}  Theorem~\ref{thrm:fedprox_main_smooth} (ours)  & \cellcolor[gray]{0.77} $\mathcal{O}\left(\frac{1}{I\epsilon^{2}} + \frac{1}{\epsilon^{3/2}}\right)$ & \cellcolor[gray]{0.77}\cmark & \cellcolor[gray]{0.77}\xmark & \cellcolor[gray]{0.77}\cmark \\
\cellcolor[gray]{0.77}& \cellcolor[gray]{0.77} Theorem~\ref{thrm:fedprox_main_nonsmooth} (ours) &\cellcolor[gray]{0.77} $\mathcal{O}\left(\frac{1}{\epsilon^2}\right)$ & \cellcolor[gray]{0.77}\cmark & \cellcolor[gray]{0.77}\cmark & \cellcolor[gray]{0.77}\cmark\\
\cellcolor[gray]{0.9}\FedMSPP & \cellcolor[gray]{0.9} Theorem~\ref{thrm:fedmspp_main_smooth} (ours) & \cellcolor[gray]{0.9} $\mathcal{O}\left(\frac{1}{bI\epsilon^{2}}+\frac{1}{\epsilon^{3/2}} \right)$ & \cellcolor[gray]{0.9} \cmark & \cellcolor[gray]{0.9} \xmark & \cellcolor[gray]{0.9} \cmark\\
\cellcolor[gray]{0.9} & \cellcolor[gray]{0.9} Theorem~\ref{thrm:fedmspp_main_nonsmooth} (ours) & \cellcolor[gray]{0.9} $\mathcal{O}\left(\frac{1}{\epsilon^{2}} \right)$ & \cellcolor[gray]{0.9} \cmark & \cellcolor[gray]{0.9} \cmark & \cellcolor[gray]{0.9} \cmark\\
\hline
\FedAvg &~\citet{karimireddy2020scaffold} & $\mathcal{O}\left(\frac{1}{bI\epsilon^2} + \frac{1}{\epsilon^{3/2}}+ \frac{1}{\epsilon}\right)$ & \xmark & \xmark & \cmark  \\
&\citet{yu2019parallel} & $\mathcal{O}\left(\frac{1}{bM\epsilon^2} + \frac{Mb}{\epsilon}\right)$  & \xmark & \xmark & \xmark \\
&~\citet{khanduri2021stem} & $\mathcal{O}\left(\frac{1}{\epsilon^{3/2}}\right)$ & \xmark & \xmark & \xmark \\
\hline
\SCAFFOLD &~\citet{karimireddy2020scaffold} & $\mathcal{O}\left(\frac{1}{bI\epsilon^2} + \frac{(M/I)^{2/3}}{\epsilon}\right)$ & \cmark & \xmark & \cmark\\
\hline
\FedPD &~\citet{zhang2020fedpd} & $\mathcal{O}\left(\frac{1}{\epsilon}\right)$ & \xmark & \xmark & \xmark \\
\hline
\STEM &~\citet{khanduri2021stem} & $\mathcal{O}\left(\frac{1}{\epsilon}\right)$ & \xmark & \xmark & \xmark \\
\hline
\FCO &~\citet{yuan2021federated} &  \tabincell{c}{$\mathcal{O}\left(\frac{1}{bM\epsilon^2} + \frac{1}{\epsilon}\right)$ \\ (convex composite) }  & \cmark & \cmark & \xmark \\
\hline
%\FedSplit &~\citep{pathak2020fedsplit} &  \tabincell{c}{$\mathcal{O}\left(\frac{1}{\sqrt{\epsilon}} \right)$ \\ (convex) }  & \cmark & \xmark & \xmark \\
%\hline
\end{tabular}
\caption{Comparison of heterogeneous FL algorithms in terms of communication complexity bounds for reaching an $\epsilon$-stationary solution, independence of local dissimilarity (LD), applicability to non-smooth (NS) functions and tolerance to partial participation (PP). Except for \FCO, all the results listed are for non-convex functions. The involved of quantities are $M$: total number of devices; $I$: number of chosen devices for partial participation; $b$: minibatch size for local stochastic optimization. \label{tab:result_comparison}}
\end{table}

\textbf{Comparison with prior results.} In Table~\ref{tab:result_comparison}, we summarize our communication complexity bounds for \FedProx~(\FedMSPP) and compare them with several related heterogeneous FL algorithms in terms of the dependency on local dissimilarity, applicability to non-smooth problems and tolerance to partial participation. A few observations are in order. \emph{First}, regarding the requirement of local dissimilarity, all of our $\mathcal{O}(\frac{1}{\epsilon^2})$ bounds are independent of local dissimilarity conditions, and they are comparable to those of \SCAFFOLD~and \FCO~(for convex problems) which are also invariant to local dissimilarity. \emph{Second}, with regard to the applicability to non-smooth optimization, our convergence guarantees in Theorem~\ref{thrm:fedprox_main_nonsmooth} and Theorem~\ref{thrm:fedmspp_main_nonsmooth} are established for non-smooth and weakly convex functions. While \FCO~is the only one in the other considered algorithms that can be applied to non-smooth problems, it is customized for federated convex composite optimization with potentially non-smooth regularizers~\citep{yuan2021federated}. \emph{Third}, in terms of tolerance to partial participation, all of our results are robust to device sampling, and the $\mathcal{O}(\frac{1}{bI\epsilon^2})$ bound in Theorem~\ref{thrm:fedmspp_main_smooth} for \FedMSPP~is comparable to the best known results under partial participation as achieved by \FedAvg~and \SCAFFOLD. If assuming that all the devices participate in local update for each communication round and under certain local dissimilarity conditions, substantially faster $\mathcal{O}(\frac{1}{\epsilon})$ bounds are possible for \STEM~and \FedPD, while the $\mathcal{O}(\frac{1}{\epsilon^{3/2}})$ bounds can be achieved by \FedAvg~\citep{khanduri2021stem}. To summarize the comparison, our local dissimilarity invariant convergence bounds for \FedProx~(\FedMSPP) are comparable to the best-known rates in the identical setting, while covering the generic non-smooth and non-convex cases which to our knowledge so far has not been possible for other FL algorithms.

\newpage

\noindent\textbf{Highlight of theoretical contributions:}
\begin{itemize}%[leftmargin=*]
  \item From the perspective of algorithmic stability theory, we provide a set of novel local dissimilarity invariant convergence guarantees for the widely used \FedProx~algorithm for non-convex heterogeneous FL, with smooth or non-smooth local functions. Our theory for the first time reveals that local dissimilarity and smoothness are not necessary to guarantee the convergence of \FedProx~with reasonable rates.
  \item We present \FedMSPP~as a minibatch stochastic extension of \FedProx~and analyze its convergence behavior in terms of population optimality, again without assuming any type of local dissimilarity conditions. The main result provably shows that \FedMSPP~converges favorably for both smooth and non-smooth objectives, while enjoying linear speedup in terms of minibatching size and partial participation ratio for smooth problems.
\end{itemize}
Finally, while the main contribution of this work is essentially theoretical, we have also carried out a preliminary numerical study on several benchmark FL datasets to corroborate our theoretical findings about the improved sample efficiency of \FedMSPP.

\vspace{0.2in}
\noindent\textbf{Paper organization.} In Section~\ref{sect:analysis_fedprox} we present our local dissimilarity invariant convergence analysis for the vanilla \FedProx~with smooth or non-smooth loss functions. In Section~\ref{sect:analysis_fedmspp} we propose \FedMSPP~as a minibatch stochastic extension of \FedProx~and analyze its convergence behavior through the lens of algorithmic stability theory. In Section~\ref{apdx:related_work}, we present some additional related work on the topics covered by this paper. In Section~\ref{apdx:experiments}, we present a preliminary experimental study on the convergence behavior of \FedMSPP. The concluding remarks are made in Section~\ref{sect:conclusion}. The technical proofs are relegated to the appendix sections.

%\newpage\clearpage

\section{Convergence of \FedProx}
\label{sect:analysis_fedprox}

We begin by providing an improved analysis for the vanilla \FedProx~independent of the local dissimilarity type conditions. We first introduce notations that will be used in the analysis to follow.

\textbf{Notations.} Throughout the paper, we use $[n]$ to denote the set $\{1,...,n\}$, $\|\cdot\|$ to denote the Euclidean norm and $\langle\cdot, \cdot\rangle$ to denote the Euclidean inner product. We say a function $f$ is $G$-Lipschitz continuous if $|f(w) - f(w')|\le G\|w - w'\|$ for all $w, w'\in \mathbb{R}^p$, and it is $L$-smooth if $|\nabla f(w) - \nabla f(w')|\le L\|w - w'\|$ for all $w, w' \in \mathbb{R}^p$. Moreover, we say $f$ is $\nu$-weakly convex if for any $w,w' \in \mathbb{R}^p$,
\[
f(w) \ge f(w') + \langle \partial f(w'), w - w'\rangle - \frac{\nu}{2} \|w-w'\|^2,
\]
where $\partial f(w')$ represents a subgradient of $f$ evaluated at $w'$. We denote by
\[
f_{\eta}(w):=\min_{u} \left\{f(u) + \frac{1}{2\eta} \|u-w\|^2\right\}
\]
the $\eta$-Moreau-envelope of $f$, and by
\[
\prox_{\eta f}(w):=\argmin_{u} \left\{f(u) + \frac{1}{2\eta} \|u-w\|^2\right\}
\]
the proximal mapping associated with $f$. We also need to access the following definition of inexact local update oracle for \FedProx.
\begin{definition}[Local inexact oracle of \FedProx]\label{def:inexact_oracle_fedprox}
Suppose that the local proximal point regularized objective $Q^{(m)}_{\erm}(w; w_{t-1})$ (cf.~\eqref{equat:fedprox_iteration}) admits a global minimizer. For each time instance $t$, we say that the local update oracle of \FedProx~is $\varepsilon_t$-inexactly solved with sub-optimality $\varepsilon_t\ge 0$ if
\[
 Q^{(m)}_{\erm}(w^{(m)}_t; w_{t-1}) \le  \min_{w} Q^{(m)}_{\erm}(w; w_{t-1}) + \varepsilon_t.
\]
\end{definition}
%\begin{remark}
%Similar inexactness is considered in the original analysis of \FedProx~\citep{li2020federatedprox}.
%\end{remark}
We conventionally assume that the objective value gap $\bar \Delta_{\erm}:= \bar R_{\erm}(w_0) - \min_{w\in \mathbb{R}^p} \bar R_{\erm}(w)$ is bounded.

\subsection{Results for Smooth Problems}

The following theorem is our main result on the convergence rate of \FedProx~for smooth and non-convex federated optimization problems.

\begin{theorem}\label{thrm:fedprox_main_smooth}
Assume that for each $m\in [M]$, the loss function $\ell^{(m)}$ is $G$-Lipschitz and $L$-smooth with respect to its first argument. Set $|I_t|\equiv I$ and $\eta_t \equiv \frac{1}{3L}\min\left\{\frac{1}{T^{1/3}},\sqrt{\frac{I}{T}}\right\}$. Suppose that the local update oracle of \FedProx~is $\varepsilon_t$-inexactly solved with $\varepsilon_t\le \min\left\{\frac{G}{2L\sqrt{I}}, \frac{G\eta_t}{I}\right\}$. Let $t^*$ be an index uniformly randomly chosen in $\{0,1,...,T-1\}$. Then,
\[
 \mathbb{E}\left[\left\|\nabla \bar R_{\erm}(w_{t^*})\right\|^2\right] \lesssim \left(L\bar \Delta_{\erm} + G^2\right) \max\left\{\frac{1}{T^{2/3}}, \frac{1}{\sqrt{TI}}\right\}.
\]
\end{theorem}
\begin{proof}
A proof of this result is deferred to Appendix~\ref{ssect:proof_fedprox_main_smooth}.
\end{proof}
A few remarks are in order.
\begin{remark}
Compared to the $\mathcal{O}(\frac{1}{T})$ bound from~\citet{li2020federatedprox}, our rate established in Theorem~\ref{thrm:fedprox_main_smooth} is slower but it is valid without assuming the unrealistic $(B,0)$-LGD conditions and imposing strong regularization conditions on $I$~\citep[see, e.g., ][Remark 5]{li2020federatedprox}. Moreover, the dominant term $\frac{1}{\sqrt{TI}}$ in our bound reveals the benefit of device sampling for linear speedup which is not clear in the original analysis of~\citet{li2020federatedprox}.
\end{remark}
\begin{remark}\label{remark:fedprox_smooth_fullpar}
In the extreme case of full device participation, i.e., $I_t \equiv [M]$, the terms related to $I$ in Theorem~\ref{thrm:fedprox_main_smooth} can be removed and thus the convergence rate becomes $\frac{1}{T^{2/3}}$ under $\eta_t=\mathcal{O}(\frac{1}{LT^{1/3}})$. In this same setting, we comment that the rate can also be improved to $\mathcal{O}(\frac{1}{T})$ using our proof augments if $(B,0)$-LGD is additionally assumed.
\end{remark}
\begin{remark}
The $G$-Lipschitz-loss assumption in Theorem~\ref{thrm:fedprox_main_smooth} can be alternatively replaced by the bounded gradient condition as commonly used in the analysis of FL algorithms~\citep{li2020federatedprox,zhang2020fedpd}. Despite that our analysis does not explicitly access to any local dissimilarity conditions, the assumed $G$-Lipschitz (or bounded gradient) condition actually implies that the local objective gradients are not too dissimilar, which shares a close spirit to the typically assumed $(0,H)$-LGD condition~\citep{karimireddy2020scaffold} and inter-client-variance condition~\citep{khanduri2021stem}. It is noteworthy that these mentioned client heterogeneity conditions are substantially milder than the $(B,0)$-LGD condition as required in the original analysis of \FedProx.
\end{remark}
\subsection{Results for Non-smooth Problems}

Now we turn to study the convergence of \FedProx~for weakly convex but not necessarily smooth problems. For the sake of presentation clarity, we work on the exact \FedProx~in which the local update oracle is assumed to be exactly solved, i.e. $\varepsilon_t\equiv0$. Extension to the inexact case is more or less straightforward, though with somewhat more involved perturbation treatment. We assume that the objective value gap $\bar\Delta_{\erm,\rho}:=\bar R_{\erm,\rho}(w_{0}) - \min_w \bar R_{\erm,\rho}(w)$ associated with $\rho$-Moreau-envelope of $\bar R_{\erm}$ is bounded. The following is our main result on the convergence of \FedProx~for non-smooth and weakly convex problems.
\begin{theorem}\label{thrm:fedprox_main_nonsmooth}
Assume that for each $m\in [M]$, the loss function $\ell^{(m)}$ is $G$-Lipschitz and $\nu$-weakly convex with respect to its first argument. Set $\eta_t \equiv \frac{\rho}{\sqrt{T}}$ for arbitrary $\rho<\frac{1}{2\nu}$. Suppose that the local update oracle of \FedProx~is exactly solved with $\varepsilon_t\equiv0$. Let $t^*$ be an index uniformly randomly chosen in $\{0,1,...,T-1\}$. Then it holds that
\[
 \mathbb{E}\left[\left\|\nabla \bar R_{\erm, \rho}(w_{t^*})\right\|^2\right] \lesssim \frac{\bar \Delta_{\erm,\rho} +\rho G^2}{ \rho\sqrt{T}}.
\]
\end{theorem}
\begin{proof}
The proof technique is inspired by the arguments from~\citet{davis2019stochastic} developed for analyzing stochastic model-based algorithms, with several new elements along developed for handling the challenges introduced by the model averaging and partial participation mechanisms associated with \FedProx. A particular crux here is that due to the random subset model aggregation of $w_t=\frac{1}{|I_t|}\sum_{\xi\in I_t} w^{(\xi)}_t$, the local function values $R^{(\xi)}_{\erm}(w_t)$ are no longer independent to each other though $\xi$ is uniformly random. As a consequence, $\frac{1}{|I_t|}\sum_{\xi\in I_t} R^{(\xi)}_{\erm}(w_t)$ is \emph{not} an unbiased estimation of $\bar R_{\erm}(w_t)$. To overcome this technical obstacle, we make use of a key observation that $w^{(m)}_t$ will be almost surely close enough to $w_{t-1}$ if the learning rate $\eta_t$ is small enough (which is the case in our choice of $\eta_t$), and thus we can replace the former with the latter whenever beneficial but without introducing too much approximation error. A full proof of this result can be found in Appendix~\ref{ssect:fedprox_main_nonsmooth_proof}.
\end{proof}

A few comments are in order.
\begin{remark}\label{remark:fedprox_nonsmooth_1}
To our best knowledge, Theorem~\ref{thrm:fedprox_main_nonsmooth} is the first convergence guarantee for FL algorithms applicable to generic non-smooth and weakly convex problems. This is in sharp contrast with \FCO~\citep{yuan2021federated} which focuses on composite convex and non-smooth problems such as $\ell_1$-estimation, or \FedHT~\citep{tong2020federated} which is specially customized for cardinality-constrained sparse learning problems where the non-convexity essentially arises from the $\ell_0$-constraint.
\end{remark}

\begin{remark}\label{remark:fedprox_nonsmooth_2}
Let us consider $\bar w_{t^*}:=\prox_{\rho \bar R_{\erm}}(w_{t^*})$, the proximal mapping of $w_{t^*}$ associated with $\bar R_{\erm}$. In view of a feature of Moreau envelope to characterize stationarity~\citep{davis2019stochastic}, if $w_{t^*}$ has small gradient norm $\left\|\nabla \bar R_{\erm, \rho}(w_{t^*})\right\|$, then $\bar w_{t^*}$ must be a near-stationary solution and $w_{t^*}$ stays in the proximity of $\bar w_{t^*}$ due to the identity $\|w_{t^*} - \bar w_{t^*}\|=\rho \left\|\nabla \bar R_{\erm, \rho}(w_{t^*})\right\|$. Therefore, the bound in Theorem~\ref{thrm:fedprox_main_nonsmooth} suggests that in expectation $\bar w_{t^*}$ converges to a stationary solution and $w_{t^*}$ converges to $\bar w_{t^*}$, both at the rate of $\mathcal{O}(\frac{1}{\sqrt{T}})$.
\end{remark}

\newpage

\begin{remark}\label{remark:fedprox_nonsmooth_3}
We comment that the bound in Theorem~\ref{thrm:fedprox_main_nonsmooth} is not dependent on $I$, the number of selected devices. For $I=\mathcal{O}(1)$ and sufficiently large $T> \mathcal{O}(I^3)$, the bounds Theorem~\ref{thrm:fedprox_main_smooth} and Theorem~\ref{thrm:fedprox_main_nonsmooth} are comparable to each other, which demonstrates that the smoothness is not must-have for \FedProx~to get sharper convergence bound with small device sampling rate. However, in the near-full participation setting where $I=\mathcal{O}(M)$, the bound in Theorem~\ref{thrm:fedprox_main_nonsmooth} for non-smooth problems will be slower when $M$ is large. Extremely when $I_t=[M]$, the $\mathcal{O}(\frac{1}{\sqrt{T}})$ bound is substantially inferior to the smooth case which has improved rate of $\mathcal{O}(\frac{1}{T^{2/3}})$ as discussed in Remark~\ref{remark:fedprox_smooth_fullpar}.
\end{remark}

\section{Convergence of \FedProx~with Stochastic Minibatching}
\label{sect:analysis_fedmspp}

When it comes to the implementation of \FedProx, a notable challenge is that the local proximal point update oracle~\eqref{equat:fedprox_iteration} is by itself a full-batch ERM problem which would be expensive to solve even approximately in large-scale settings. Moreover, in the settings where the data distribution over devices is highly imbalanced, the computational overload of local update could vary significantly across the network, which impairs communication efficiency. It is thus desirable to seek stochastic approximation schemes for hopefully improving the local oracle update efficiency and overload balance of \FedProx. To this end, inspired by the recent success of minibatch stochastic proximal point methods (MSPP)~\citep{asi2020minibatch,deng2021minibatch}, we propose to implement \FedProx~using MSPP as the local stochastic optimization oracle. More precisely, let $B^{(m)}_t=\{z^{(m)}_{i,t}\}_{i=1}^b \overset{\text{i.i.d.}}{\sim} (\mathcal{D}^{(m)})^b$ be a minibatch of $b$ i.i.d. samples drawn from the distribution $\mathcal{D}^{(m)}$ at device $m$ and time instance $t\ge 1$. We denote
\begin{equation}\label{equat:local_risk_fedmspp}
R^{(m)}_{B^{(m)}_t}(w):=\frac{1}{b}\sum_{i=1}^b \ell^{(m)}(w; z^{(m)}_{i,t})
\end{equation}
as the local minibatch empirical risk function over $B^{(m)}_t$. The only modification we propose to make here is to replace the empirical risk $R^{(m)}_{\erm}(w)$ in the original update form~\eqref{equat:fedprox_iteration} with its minibatch counterpart $R^{(m)}_{B^{(m)}_t}(w)$. The resultant FL framework, which we refer to as  \FedMSPP~(Federated MSPP), is outlined in Algorithm~\ref{alg:fedmspp}. Clearly, the vanilla \FedProx~is a special case of \FedMSPP~when applied to the federated ERM form~\eqref{equat:problem_erm} with full data batch $B^{(m)}_t\equiv D^{(m)}$.

\begin{algorithm}[t]\caption{\texttt{FedMSPP}: Federated Minibatch Stochastic Proximal Point}
\label{alg:fedmspp}
\SetKwInOut{Input}{Input}\SetKwInOut{Output}{Output}\SetKw{Initialization}{Initialization}
\Input{Minibatch size $b$; learning rates $\{\gamma_t\}_{t\in [T]}$. }
\Output{$w_T$.}
\Initialization{Set $w_0$, e.g., typically as a zero vector. }

\For{$t=1, 2, ...,T$}{

\tcc{\textbf{Device selection and model broadcast on the server}}

Server uniformly randomly selects a subset $I_t\subseteq [M]$ of devices and sends $w_{t-1}$ to all the selected devices;

\tcc{\textbf{Local model updates on the selected devices}}
\For{$\xi \in I_t$ in parallel}{
Device $\xi$ samples a minibatch $B^{(\xi)}_t=\{z^{(\xi)}_{i,t}\}_{i=1}^b \overset{\text{i.i.d.}}{\sim} (\mathcal{D}^{(\xi)})^b$.

Device $\xi$ inexactly updates the its local model as
\begin{equation}\label{equat:fedmspp_iteration}
w^{(\xi)}_t \approx \argmin_{w\in \mathcal{W}} \left\{Q^{(\xi)}_{B^{(\xi)}_t}(w; w_{t-1}):=R^{(\xi)}_{B^{(\xi)}_t}(w) + \frac{1}{2\eta_t}\|w - w_{t-1}\|^2\right\},
\end{equation}
where $R^{(\xi)}_{B^{(\xi)}_t}(w)$ is given by~\eqref{equat:local_risk_fedmspp}.

Device $\xi$ sends $w^{(\xi)}_t$ back to server.
}
\tcc{\textbf{Model aggregation on the server}}
Sever aggregates the local models received from $I_t$ to update the global model as $w_t = \frac{1}{|I_t|}\sum_{\xi\in I_t} w^{(\xi)}_t$.
}
\end{algorithm}

\subsection{Results for Smooth Problems}

We first analyze the convergence rate of \FedMSPP~for smooth and non-convex problems using the tools borrowed from algorithmic stability theory. Analogous to the Definition~\ref{def:inexact_oracle_fedprox}, we introduce the following definition of inexact local update oracle for \FedMSPP.
\begin{definition}[Local inexact oracle of \FedMSPP]\label{def:inexact_oracle_fedmspp}
Suppose that the local proximal point regularized objective $Q^{(m)}_{B^{(m)}_t}(w; w_{t-1})$ (cf.~\eqref{equat:fedmspp_iteration}) admits a global minimizer. For each time instance $t$, we say that the local update oracle of \FedMSPP~is $\varepsilon_t$-inexactly solved with sub-optimality $\varepsilon_t\ge 0$ if
\[
 Q^{(m)}_{B^{(m)}_t}(w^{(m)}_t; w_{t-1}) \le  \min_{w} Q^{(m)}_{B^{(m)}_t}(w; w_{t-1}) + \varepsilon_t.
\]
\end{definition}

We also assume that the population value gap $\bar \Delta = \bar R(w^{(0)}) - \min_{w\in \mathbb{R}^p} \bar R(w)$ is bounded. The following theorem is our main result on \FedMSPP~for smooth and non-convex FL problems.

\begin{theorem}\label{thrm:fedmspp_main_smooth}
Assume that for each $m\in [M]$, the loss function $\ell^{(m)}$ is $G$-Lipschitz and $L$-smooth with respect to its first argument. Set $|I_t|\equiv I$ and $\eta_t \equiv \frac{1}{8L}\min\left\{\frac{1}{T^{1/3}},\sqrt{\frac{bI}{T}}\right\}$. Suppose that the local update oracle of \FedMSPP~is $\varepsilon_t$-inexactly solved with $\varepsilon_t \le \min\left\{\frac{G}{2L}, \frac{G^2\eta_t}{8b^2}, \frac{G\eta_t}{2bI}\right\}$. Then we have
 \[
\frac{1}{T}\sum_{t=0}^{T-1} \mathbb{E} \left[ \|\nabla \bar R(w_{t})\|^2 \right] \lesssim \left(L\bar \Delta + G^2\right) \max\left\{\frac{1}{T^{2/3}}, \frac{1}{\sqrt{TbI}}\right\}.
\]
\end{theorem}
\begin{proof}
Let us consider $ d^{(m)}_t = \nabla R^{(m)}_{B^{(m)}_t}(w^{(m)}_t)$ which is roughly the local update direction on device $m$, in the sense that $ w^{(m)}_t\approx w_{t-1} - \eta_t d^{(m)}_t $ given that the local update oracle is solved to sufficient accuracy. As a key ingredient of our proof, we show via some extended uniform stability arguments in terms of gradients (see Lemma~\ref{lemma:stability_genalization_first_order_moment}) that the averaged directions $d_t := \frac{1}{|I_t|}\sum_{\xi \in I_t} d^{(\xi)}_t$ aligns well with the global gradient $\nabla \bar R (w_{t-1})$ in expectation (see Lemma~\ref{lemma:concentrain_key}). Therefore, in average it roughly holds that $w_t = \frac{1}{|I_t|}\sum_{\xi \in I_t} w^{(\xi)}_t \approx w_{t-1} - \eta_t d_t \approx w_{t-1} -\eta_t  \nabla \bar R (w_{t-1})$, which suggests that $w_t$ is updated roughly along the direction of global gradient descent and thus guarantees quick convergence. Based on this novel analysis, we are free of imposing any kind of local dissimilarity conditions on local objectives. See Appendix~\ref{ssect:proof_fedmspp_main_smooth} for a full proof of this result.
\end{proof}
\begin{remark}
For $T\ge (bI)^3$, the bound in Theorem~\ref{thrm:fedmspp_main_smooth} is dominated by $\mathcal{O}(\frac{1}{\sqrt{TbI}})$ which gives the communication complexity $\frac{1}{bI\epsilon^2}$. This shows that \FedMSPP~enjoys linear speedup both in the size of local minibatching and in the size of device sampling.
\end{remark}
\begin{remark}
While the bound in Theorem~\ref{thrm:fedmspp_main_smooth} is derived for the population form of FL in~\eqref{equat:problem}, identical bound naturally holds for the empirical form~\eqref{equat:problem_erm} under minibatch sampling according to local data empirical distribution.
\end{remark}

\subsection{Results for Non-smooth Problems}

Analogues to \FedProx~, we can further show that \FedMSPP~converges reasonably well when applied to weakly convex and non-smooth problems. We assume that the objective value gap $\bar\Delta_{\rho}:=\bar R_{\rho}(w_{0}) - \min_w \bar R_{\rho}(w)$ associated with $\rho$-Moreau-envelope of $\bar R$ is bounded. The following is our main result in this line.

\begin{theorem}\label{thrm:fedmspp_main_nonsmooth}
Assume that for each $m\in [M]$, the loss function $\ell^{(m)}$ is $G$-Lipschitz and $\nu$-weakly convex with respect to its first argument. Set $\eta_t \equiv \frac{\rho}{\sqrt{T}}$ for arbitrary $\rho<\frac{1}{2\nu}$. Suppose that the local update oracle of \FedMSPP~is exactly solved with $\varepsilon_t\equiv0$. Let $t^*$ be an index uniformly randomly chosen in $\{0,1,...,T-1\}$. Then it holds that
\[
 \mathbb{E}\left[\left\|\nabla \bar R_{\erm, \rho}(w_{t^*})\right\|^2\right] \lesssim \frac{\bar \Delta_{\rho} +\rho G^2}{ \rho\sqrt{T}}.
\]
\end{theorem}
\begin{proof}
The proof argument is a slight adaptation of that of Theorem~\ref{thrm:fedprox_main_nonsmooth} to the population FL setup~\eqref{equat:problem} with \FedMSPP. For the sake of completeness, a full proof is reproduced in Appendix~\ref{ssect:fedmspp_main_nonsmooth_proof}.
\end{proof}
We comment in passing that the discussions made in Remarks~\ref{remark:fedprox_nonsmooth_1}-\ref{remark:fedprox_nonsmooth_3} immediately extend to Theorem~\ref{thrm:fedmspp_main_nonsmooth}.

\section{Additional Related Work}
\label{apdx:related_work}

\textbf{Heterogenous federated learning.} The presence of device heterogeneity features a key distinction between FL and classic distributed learning. The most commonly used FL method is \FedAvg~\citep{mcmahan2017communication}, where the local update oracle is formed as multi-epoch SGD. \FedAvg~was early analyzed for identical functions~\citep{stich2019local,stich2020error} under the name of local SGD. In heterogeneous setting, numerous recent studies have focused on the analysis of \FedAvg~and other variants under various notions of local dissimilarity~\citep{chen2020toward,khaled2020tighter,li2020convergence,woodworth2020minibatch,reddi2021adaptive,li2022distributed}. As another representative FL method, \FedProx~\citep{li2020federatedprox} has recently been proposed to apply averaged proximal point updates to solve heterogeneous federated minimization problems. The theoretical guarantees of \FedProx~have been established for both convex and non-convex problems, but under a fairly stringent assumption of gradient similarity to measure data heterogeneity~\citep{li2020federatedprox,nguyen2020fast,pathak2020fedsplit}. This assumption was relaxed by \FedPD~\citep{zhang2020fedpd} inside a meta-framework of primal-dual optimization. The \SCAFFOLD~\citep{karimireddy2020scaffold} and \VRLSGD~\citep{liang2019variance} are two algorithms that utilize variance reduction techniques to correct the local update directions, achieving convergence guarantees independent of the data heterogeneity. For composite non-smooth FL problems, the \FCO~proposed in~\citet{yuan2021federated} employs a server dual averaging procedure to circumvent the curse of primal averaging suffered by \FedAvg. In sharp contrast to these prior works which either require local dissimilarity conditions, or require full device participation, or only applicable to smooth problems, we show through a novel analysis based on algorithmic stability theory that the well-known \FedProx~can actually overcome these shortcomings simultaneously.

\textbf{Minibatch stochastic proximal point methods.} The proposed \FedMSPP~algorithm is a variant of \FedProx~that simply replaces the local proximal point oracle with MSPP, which in each iteration updates the local model via (approximately) solving a proximal point estimator over a stochastic minibatch. The MSPP-type methods have been shown to attain a substantially improved iteration stability and adaptivity for large-scale machine learning, especially in non-smooth optimization settings~\citep{li2014efficient,wang2017memory,asi2019stochastic,deng2021minibatch}. However, it is not yet known if \FedProx~or \FedMSPP~can achieve similar strong guarantees for non-smooth heterogenous FL problems.

\textbf{Algorithmic stability.} Our analysis for \FedMSPP~builds largely upon the classic algorithmic stability theory. Since the fundamental work of~\citet{bousquet2002stability}, algorithmic stability has been serving as a powerful proxy for establishing strong generalization bounds~\citep{zhang2003leave,mukherjee2006learning,shalev2010learnability}. Particularly, the state-of-the-art risk bounds of strongly convex ERM are offered by approaches based on the notion of uniform stability~\citep{feldman2018generalization,feldman2019high,bousquet2020sharper,klochkov2021stability}. It was shown by~\citet{hardt2016train} that the solution obtained via SGD is expected to be unformly stable for smooth convex or non-convex loss functions. For non-smooth convex losses, the stability induced generalization bounds have been established for SGD~\citep{bassily2020stability,lei2020fine}. Through the lens of algorithmic stability theory, convergence rates of MSPP have been studied for non-smooth and convex~\citep{wang2017memory} or weakly convex~\citep{deng2021minibatch} losses.

\section{Experimental Results}
\label{apdx:experiments}

In this section, we carry out a preliminary experimental study to demonstrate the speed-up behavior of \FedMSPP~under varying minibatch sizes for achieving comparable test performances to \FedProx. We also conventionally use \FedAvg~as a baseline algorithm for comparison.

\subsection{Data and Models}

\begin{table*}[h]
\centering
\begin{tabular}{|cccc|}
\hline
Dataset & Model & \# Devices & \# Samples (Training) \\
\hline
MNIST & 2-layer CNN & 100 &  $63,000 \ (56700)$\\

FEMNIST & 2-layer CNN & 50 & $55050 \ (49545)$\\

Sent140 & 2-layer LSTM & 261 & $21546 \ (17237)$ \\
\hline
\end{tabular}
\caption{Statistics of data and models used in the experiments. \label{tab:data_statistics}}
\end{table*}

\newpage

We compare the considered algorithms over the following three benchmark datasets popularly used for evaluating heterogenous FL approaches:

\begin{itemize}%[leftmargin=*]
  \item The MNIST~\citep{lecun1998gradient} dataset of handwritten digits 0-9 is used for digit image classification with a two layer convolutional neural network (CNN). The model takes as input the images of size $28\times 28$, and first performs a 2-layer (\{1, 32, max-pooling\}, \{32, 64, max-pooling\}) convolution  followed by a fully connected (FC) layer. We use 63,000 images in which $90\%$ are for training and the rest for test. The data are distributed over 100 devices such that each device has samples of only 2 digits.
  \item The FEMNIST~\citep{li2020federatedprox} dataset is a subset of the 62-class EMNIST~\citep{cohen2017emnist} database constructed by sub-sampling 10 lower case characters ('a'-'j'). We study the performances of the considered algorithms for character image classification using the same two layer CNN as used for MNIST, which takes as input the images of size $28\times 28$. We use 55,050 images in which $90\%$ are for training and the rest for test. The data are distributed over 50 devices, each of which has samples of 3 characters.
  \item The Sent140~\citep{go2009twitter} dataset of text sentiment analysis on tweets is used for evaluating the considered algorithms for sentiment classification. The model we use is a two layer LSTM binary classifier containing 256 hidden units followed by a densely-connected layer. The input is a sequence of 25 characters represented by a 300-dimensional GloVe embedding~\citep{pennington2014glove} and the output is one character per training sample. We use for our experiment a total number of $21,546$ tweets from $261$ twitter accounts, each of which corresponds to a device. The training/test sample split is $80\%$ versus $20\%$.
\end{itemize}
The statistics of the data and models in use are summarized in Table~\ref{tab:data_statistics}.

\subsection{Implementation Details and Performance Metrics}

We generally follow the instructions of~\citet{li2020federatedprox} for implementing \FedProx, \FedMSPP~and \FedAvg. More specifically, we use SGD as the local solver for \FedProx, \FedMSPP~and \FedAvg. For \FedMSPP, we implement with three varying minibatch sizes on each dataset as shortly reported in the next subsection about results. The hyper-parameters used in our implementation, such as number of communication rounds and number of local SGD epochs, are listed in Table~\ref{tab:hyperparameters}.

\begin{table*}[h]
\centering
\begin{tabular}{|cccc|}
\hline
Hyper-parameter & MNIST & FEMNIST & Sent140 \\
\hline
\#Communication rounds & 200 & 300 &  300 \\

\#Local SGD epochs & 2 & 5 & 10\\

Local SGD minibatch size & 567 & 512 & 100 \\

Local SGD learning rate & 0.25 & 0.06 & 0.1 \\

Strength of regularization $\mu_t$ & 0.1 & 0.1 & 0.001 \\
\hline
\end{tabular}
\caption{Hyper-parameter settings. \label{tab:hyperparameters}}
\end{table*}

Since the chief goal of this empirical study is to illustrate the benefit of \FedMSPP~for speeding up the convergence of \FedProx, we use the numbers of data points and communication rounds needed for reaching the desired solution accuracy as performance metrics. The desired test accuracies are $\{80\%, 90\%,  95\%\}$ on MNIST, $\{80\%, 85\%,  91\%\}$ on FEMNIST, and $\{68\%, 70\%, 73\%\}$ on Sent140.

\begin{figure}[h!]
\centering
\subfigure[MNIST: Numbers of data points needed to reach $80\%$, $90\%$ and $95\%$ test accuracies. For \FedMSPP, we test with different minibatch sizes 81, 63, 10.]{
\includegraphics[width=4.5in]{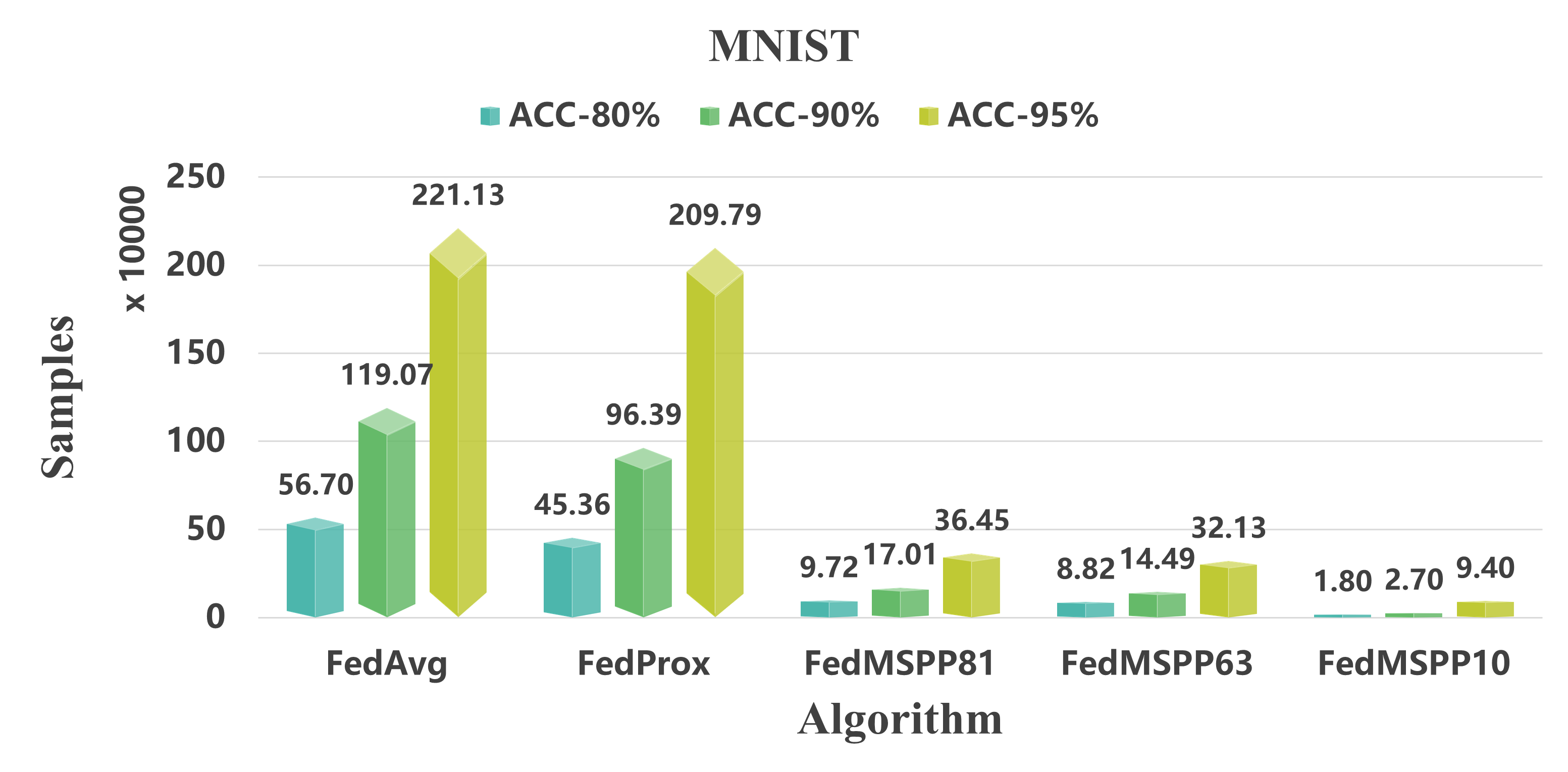}
}
\subfigure[FEMNIST: Numbers of data points needed to reach $80\%$, $85\%$ and $91\%$ test accuracies. For \FedMSPP, we test with different minibatch sizes 128, 64, 16.]{
\includegraphics[width=4.5in]{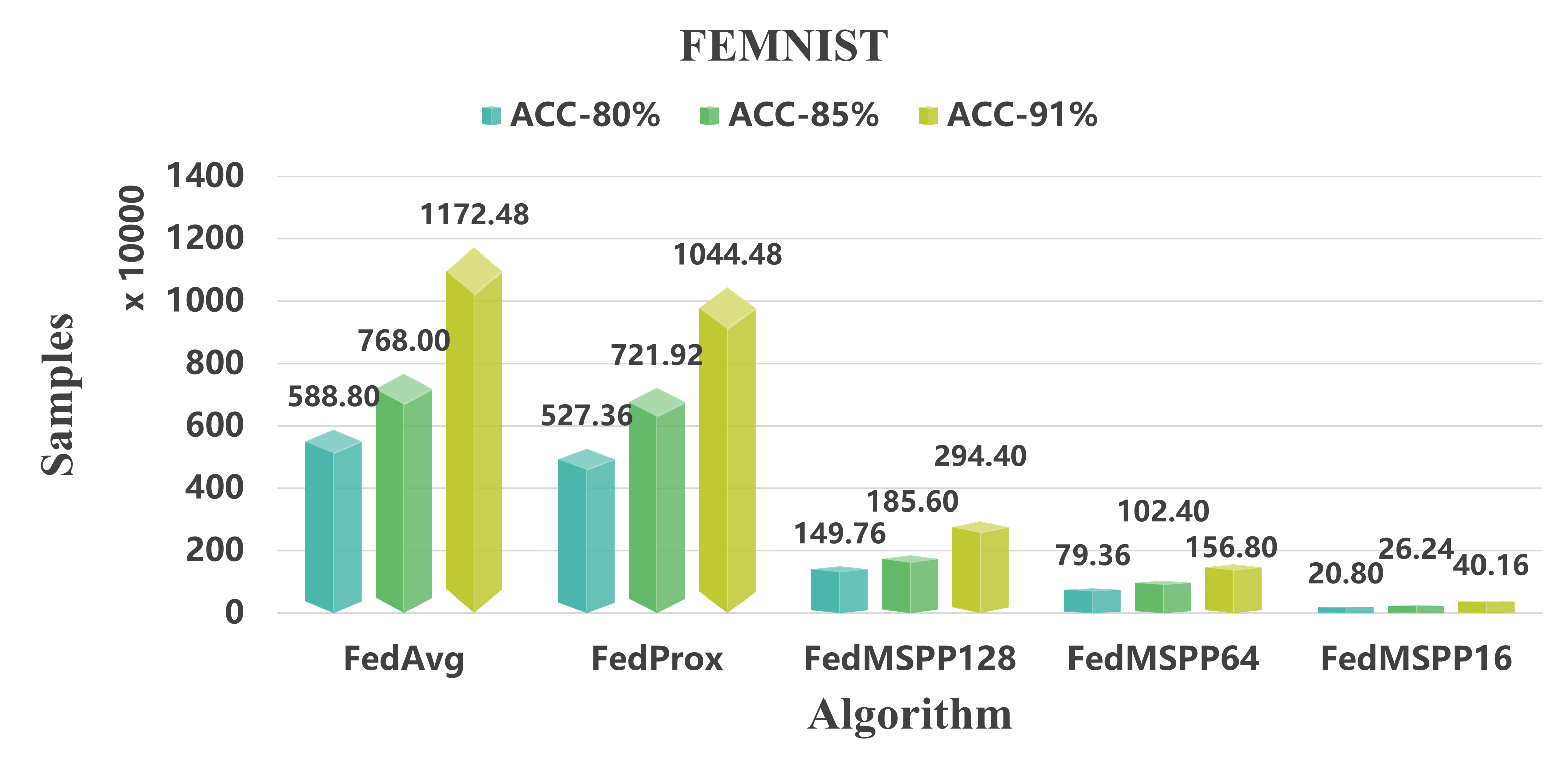}
}
\subfigure[Sent140: Numbers of data points needed to reach $68\%$, $70\%$ and $73\%$ test accuracies. For \FedMSPP, we test with different minibatch sizes 70, 50, 20.]{
\includegraphics[width=4.5in]{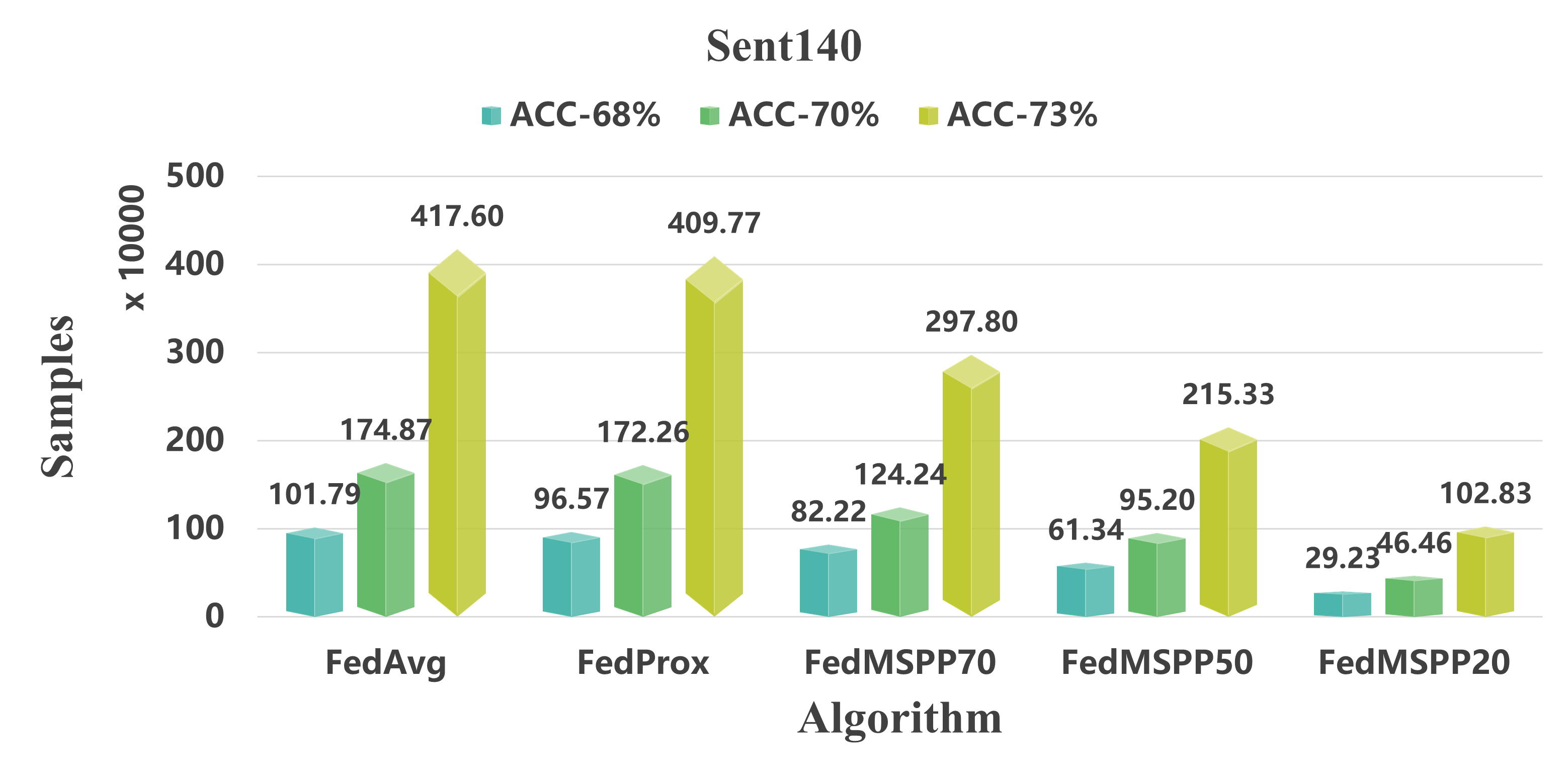}
}
\caption{Comparison of numbers of data points accessed by the considered algorithms to reach varying desired test accuracies.}
\label{fig:fedmspp_samples}
\end{figure}

\newpage

\begin{figure}[h!]
\centering

\subfigure[MNIST: Rounds of communication needed to reach $80\%$, $90\%$ and $95\%$ test accuracies. For \FedMSPP, we test with different minibatch sizes 81, 63, 10.]{
\includegraphics[width=4.5in]{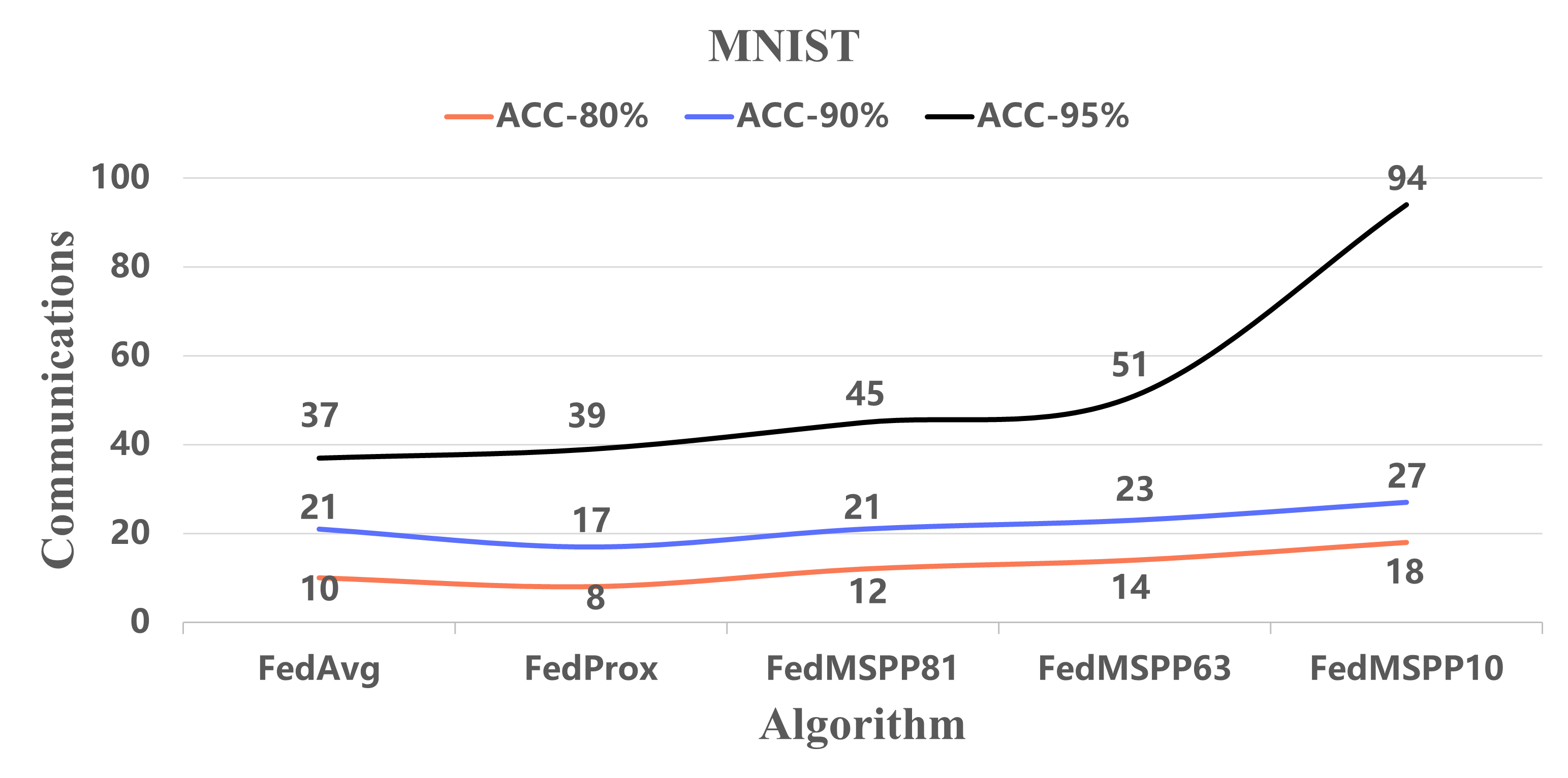}
}
\subfigure[FEMNIST: Rounds of communication needed to reach $80\%$, $85\%$ and $91\%$ test accuracies. For \FedMSPP, we test with different minibatch sizes 128, 64, 16.]{
\includegraphics[width=4.5in]{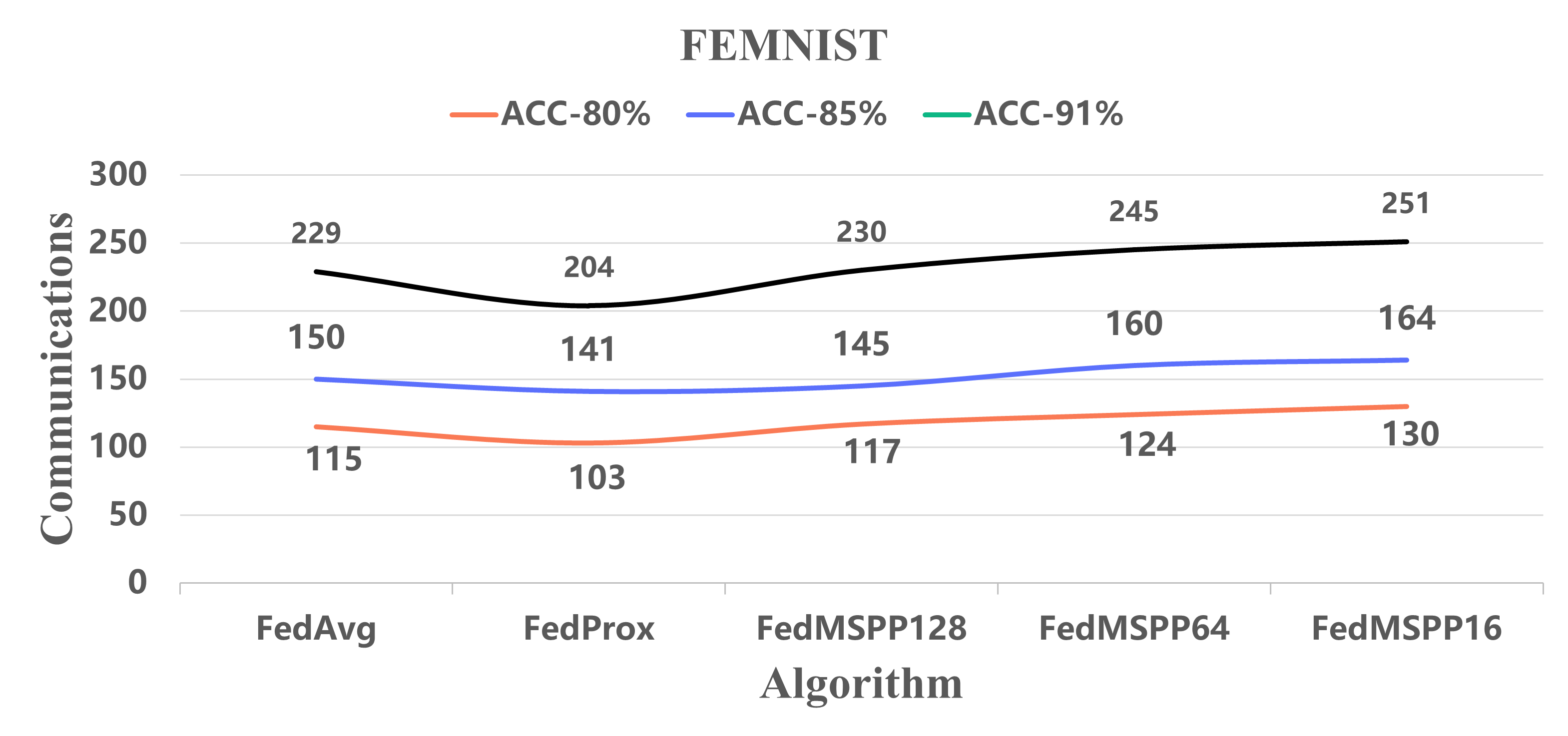}
}
\subfigure[Sent140: Rounds of communication needed to reach $68\%$, $70\%$ and $73\%$ test accuracies. For \FedMSPP, we test with different minibatch sizes 70, 50, 20.]{
\includegraphics[width=4.5in]{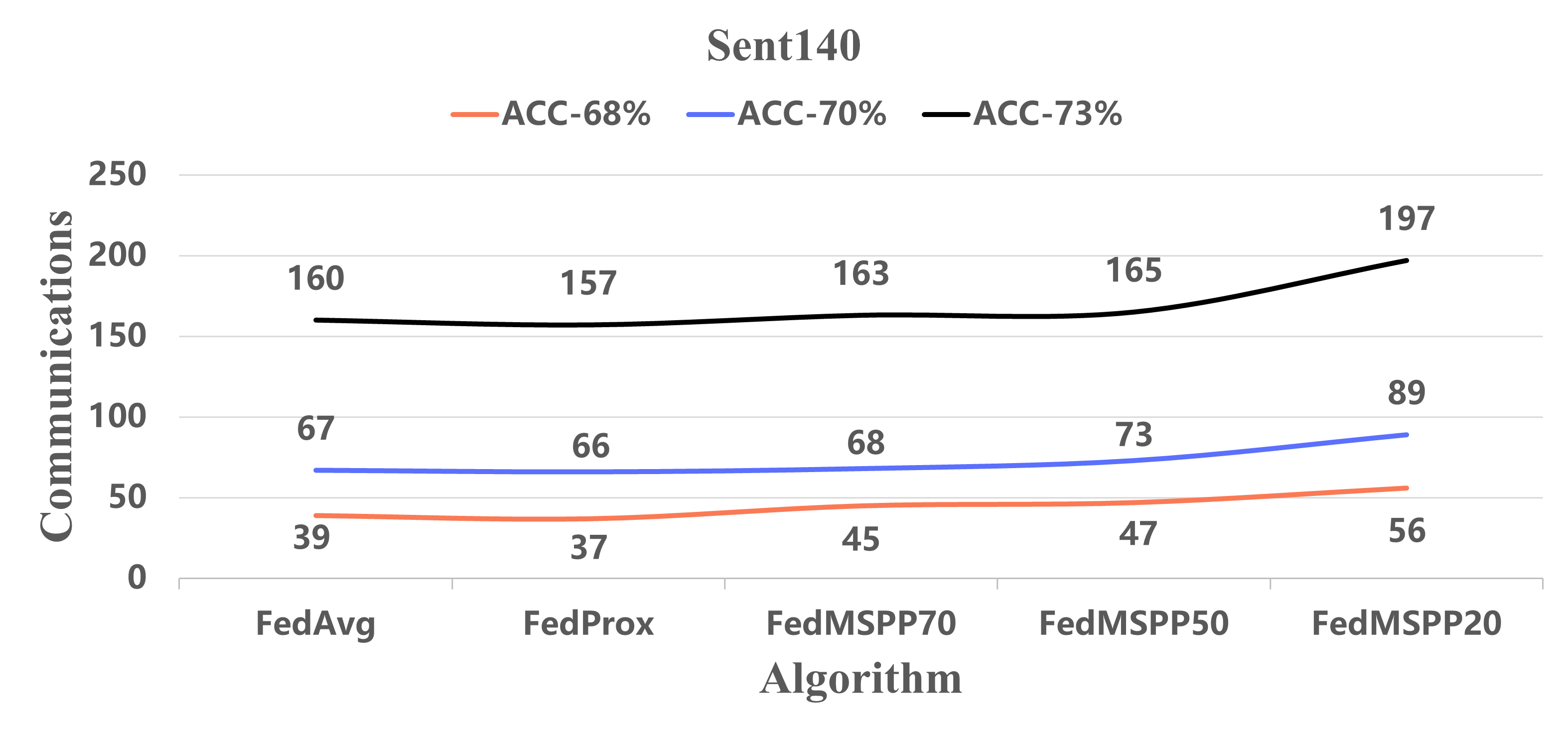}
}
\caption{Comparison of rounds of communication needed by the considered algorithms to reach varying desired test accuracies.}
\label{fig:fedmspp_communications}
\end{figure}

\newpage\clearpage

\subsection{Results}

In Figure~\ref{fig:fedmspp_samples}, we show the numbers of data samples accessed by the considered algorithms to reach comparable test accuracies. For \FedMSPP, we test with minibatch sizes $\{81, 63, 10\}$ on MNIST, $\{128, 64, 16\}$ on FEMNIST, and $\{75, 50, 20\}$ on Sent140. From this set of results we can observe that:
\begin{itemize}%[leftmargin=*]
  \item On all the three datasets in use, \FedMSPP~with varying minibatch sizes consistently needs significantly fewer samples than \FedProx~and \FedAvg~to reach the desired test accuracies.
  \item \FedMSPP~with smaller minibatch size tends to have better sample efficiency.
\end{itemize}
Figure~\ref{fig:fedmspp_communications} shows the corresponding rounds of communication needed to reach comparable test accuracies. From this group results we can see that in most cases, \FedMSPP~just needs slightly increased rounds of communication than \FedProx~and \FedAvg~to reach comparable generalization accuracy.

Overall, our numerical results confirm that \FedMSPP~can be served as a safe and computationally more efficient replacement to \FedProx~on the considered heterogenous FL tasks.

\section{Conclusion}
\label{sect:conclusion}

In this paper, we have exposed three shortcomings of the prior convergence analysis for \FedProx~in unrealistic local dissimilarity assumptions, inapplicability to non-smooth problems and expensive (and potentially imbalanced) computational cost of local update. In order to tackle these issues, we developed a novel convergence theory for the vanilla \FedProx~and its minibatch stochastic variant, \FedMSPP, through the lens of algorithmic stability theory. At nutshell, our results reveal that with minimal modifications, \FedProx~is able to kill three birds with one stone: it enjoys favorable rates of convergence which are simultaneously invariant to local dissimilarity, applicable to smooth or non-smooth problems, and scaling linearly with respect to local minibatch size and device sampling ratio for smooth problems. To the best of our knowledge, the present work is the first theoretical contribution that achieves all these appealing properties in a single FL framework.

\section*{Acknowledgement}

Xiao-Tong Yuan was funded in part by the National Key Research and Development Program of China under Grant No. 2018AAA0100400 and in part by Natural Science Foundation of China (NSFC) under Grant No.61876090, No.61936005 and No.U21B2049.

\newpage

\bibliography{mybib2}
\bibliographystyle{plainnat}

%%%%%%%%%%%%%%%%%%%%%%%%%%%%%%%%%%%%%%%%%%%%%%%%%%%%%%%%%%%%

\newpage
\appendix

\section{Preliminaries }
\label{apdx:preliminary}
We present in this section some preliminary results on the classic algorithmic stability theory to be used in our analysis. Let us consider an algorithm $A: \mathcal{Z}^N \mapsto \mathcal{W}$ that maps a training dataset $S=\{z_i\}_{i\in [N]} \in \mathcal{Z}^N$ to a model $A(S)$ in a closed subset $\mathcal{W}\subseteq \mathbb{R}^p$ such that the following population risk function (with a slight abuse of notation) evaluated at the model is as small as possible:
\[
R(A(S)):=\mathbb{E}_{Z \sim \mathcal{D}}[\ell (A(S);Z)].
\]
The corresponding empirical risk is defined by
\[
R_S(A(S)): = \mathbb{E}_{Z \sim \texttt{Unif}(S)}[\ell (A(S);Z)] = \frac{1}{N}\sum_{i=1}^N \ell(A(S); z_i).
\]
We denote by $S \doteq S'$ if a pair of datasets $S$ and $S'$ differ in a single data point. The following concept of stability that serves as a powerful tool for analyzing the generalization bounds of learning algorithms~\citep{elisseeff2005stability,hardt2016train,bassily2020stability}.
\begin{definition}[Uniform Argument Stability]\label{def:mean_uniform_stability}
Let $A: \mathcal{Z}^N \mapsto \mathcal{W}$ be a learning algorithm that maps a dataset $S \in \mathcal{Z}^N$ to a model $A(S) \in \mathcal{W}$. Then $A$ is said to have $\gamma$-uniform stability if for every $N\ge 1$,
\[
\sup_{S \doteq S'}\|A(S) - A(S')\| \le \gamma.
\]
\end{definition}

The following basic lemma is about the uniform argument stability of an inexact regularized empirical risk minimization (ERM) estimator. See Appendix~\ref{ssect:proof_stability_rERM} for its proof.
\begin{lemma}\label{lemma:stability_rERM}
Assume that the loss function $\ell$ is $G$-Lipschitz with respect to its first argument. Suppose that the regularized objective $R^r_S(w):=\frac{1}{N}\sum_{i=1}^N \ell(w; z_i) + r(w)$ is $\lambda$-strongly convex for any $S$. Consider the inexact estimator $w_S$ that satisfies the following for some $\varepsilon_t\ge 0$:
\[
 R^r_S(w_S) \le \min_{w} R^r_S(w^*_S) + \varepsilon_t.
\]
Then $w_S$ has uniform argument stability with parameter $\frac{4G}{\lambda N} + 2\sqrt{\frac{2\varepsilon_t}{\lambda}}$.
\end{lemma}

We further need to use the following variant of Efron-Stein inequality to random vector-valued functions~\citep[see, e.g., Lemma 6,][]{rivasplata2018pac}.
\begin{lemma}[Efron-Stein inequality for vector-valued functions]\label{lemma:efron_stein}
Let $S=\{Z_1, Z_2, ...,Z_N\}$ be a set of i.i.d. random variables valued in $\mathcal{Z}$. Suppose that the function $h: \mathcal{Z}^N \mapsto \mathcal{H}$ valued in a Hilbert space $\mathcal{H}$ is measurable and satisfies the bounded differences property, i.e., the following inequality holds for any $i\in [N]$ and any $z_1,...,z_N, z'_i$:
\[
\|h(z_1,...,z_{i-1}, z_i, z_{i+1}, ..., z_N) - h(z_1,...,z_{i-1}, z'_i, z_{i+1},...,z_N)\| \le \beta.
\]
Then it holds that
\[
\mathbb{E}_S \left[\left\| h(S) -  \mathbb{E}_S[h(S)]\right\|^2\right] \le \beta^2 N.
\]
\end{lemma}
Based on the Efron-Stein inequality in Lemma~\ref{lemma:efron_stein}, we can establish the following lemma which states the generalization bounds of a uniformly stable learning algorithm in terms of gradient. A proof of this result can be found in Appendix~\ref{ssect:proof_stability_genalization_first_order_moment}.
\begin{lemma}\label{lemma:stability_genalization_first_order_moment}
Suppose that a learning algorithm $A: \mathcal{Z}^N \mapsto \mathcal{W}$ has $\gamma$-uniform stability. Assume that the loss function $\ell$ is $G$-Lipschitz and $L$-smooth with respect to its first argument. Then the following bounds hold:
\[
\begin{aligned}
\left\|\mathbb{E}_S \left[\nabla R(A(S)) - \nabla R_S(A(S))\right]\right\| \le& L \gamma, \\
%\mathbb{E}_S \left[\|\nabla R(A(S)) - \nabla R_S(A(S))\|^2\right] \le& 34L^2\gamma^2 + \frac{8G^2}{N}, \\
\mathbb{E}_S \left[\left\|\nabla R(A(S)) - \mathbb{E}_S\left[\nabla R(A(S))\right]\right\|^2\right] \le& L^2\gamma^2 N.
\end{aligned}
\]
\end{lemma}

\section{Proofs for Section~\ref{sect:analysis_fedprox}}
\label{apdx:proofs_fedprox}

\subsection{Proof of Theorem~\ref{thrm:fedprox_main_smooth}}
\label{ssect:proof_fedprox_main_smooth}
Let $d^{(m)}_t = \nabla R^{(m)}_{\erm}(w^{(m)}_t)$. We define the following quantities
\begin{equation}\label{equat:d_t_bar}
d_t := \frac{1}{|I_t|}\sum_{\xi \in I_t} d^{(\xi)}_t, \quad \bar d_t := \frac{1}{M}\sum_{m=1}^M d^{(m)}_t.
\end{equation}
The following elementary lemma is useful in our analysis.
\begin{lemma}\label{lemma:inexact_key_1}
Assume that for each $m\in [M]$, the loss function $\ell^{(m)}$ is $G$-Lipschitz. Then it holds that
\[
\mathbb{E}\left[d_t\right] = \bar d_t, \ \ \ \mathbb{E}\left[\|d_t - \bar d_t\|^2\right] \le \frac{G^2}{|I_t|}.
\]
\end{lemma}
\begin{proof}
By uniform without-replacement sampling strategy we have
\[
\mathbb{E}\left[ d_{I_t}\right] = \mathbb{E}\left[\frac{1}{|I_t|}\sum_{\xi\in I_t} d^{(\xi)}_t\right] = \frac{1}{|I_t|}\sum_{\xi\in I_t} \mathbb{E}\left[d^{(\xi)}_t\right] = \frac{1}{|I_t|}\sum_{\xi\in I_t}\frac{1}{M}\sum_{m=1}^M d^{(m)}_t = \bar d_t.
\]
Then it follows that
\[
\begin{aligned}
\mathbb{E}\left[\| d_t - \bar d_t\|^2\right] =& \mathbb{E}\left[\left\|\frac{1}{|I_t|}\sum_{\xi\in I_t} d^{(\xi)}_t - \bar d_t \right\|^2\right] \\
=& \frac{1}{|I_t|^2}\mathbb{E}\left[\left\|\sum_{\xi\in I_t} (d^{(\xi)}_t - \bar d_t)\right\|^2\right] \\
=& \frac{1}{|I_t|^2}\sum_{\xi\in I_t} \mathbb{E}\left[\left\| d^{(\xi)}_t - \bar d_t \right\|^2\right]
\le& \frac{1}{|I_t|} \mathbb{E}\left[(d^{(\xi)}_t)^2\right] \le \frac{G^2}{|I_t|},
\end{aligned}
\]
where we have used the fact $\mathbb{E}\left[d^{(\xi)}_t\right] = \bar d_t$, the independence among the indices in $I_t$ and the $G$-Lipschitzness of losses. The desired bounds are proved.
\end{proof}

\newpage

We also need the following lemma which quantifies the impact of local update precision to the gradient norm at the inexact solution.
\begin{lemma}\label{lemma:inexact_key_2_fedprox}
Assume that for each $m\in [M]$, the loss function $\ell^{(m)}$ is $L$-smooth with respect to its first argument. Suppose that the local update oracle of \FedProx~is $\varepsilon_t$-inexactly solved and $\eta_t<\frac{1}{L}$. Then it holds that
\[
\left\|w^{(m)}_t - w_{t-1} + \eta_t d^{(m)}_t \right\| \le 2L\varepsilon_t\eta_t.
\]
\end{lemma}
\begin{proof}
Recall $Q^{(m)}_{\erm}(w;w_{t-1})=R^{(m)}_{\erm}(w) + \frac{1}{2\eta_t}\|w - w_{t-1}\|^2$. Since the loss functions $\ell^{(m)}$ are $L$-smooth and $\eta_t<\frac{1}{L}$, $Q^{(m)}_{\erm}(w;w_{t-1})$ is strongly convex and thus admits a global minimizer. Then we have
\[
\begin{aligned}
&\left\|\nabla R^{(m)}_{\erm}(w^{(m)}_t) + \frac{1}{\eta_t}(w^{(m)}_t - w_{t-1})\right\| \\
=& \left\|\nabla Q^{(m)}_{\erm}(w^{(m)}_t; w_{t-1})\right\|\le 2L \left(Q^{(m)}_{\erm}(w^{(m)}_t; w_{t-1}) - \min_w Q^{(m)}_{\erm}(w; w_{t-1})\right) \le 2L\varepsilon_t,
\end{aligned}
\]
where in the last inequality is due to Definition~\ref{def:inexact_oracle_fedprox}. This implies the desired bound.
\end{proof}

\vspace{0.1in}

\begin{lemma}\label{lemma:fedprox_concentrain_key}
Assume that for each $m\in [M]$, the loss function $\ell^{(m)}$ is $G$-Lipschitz and $L$-smooth with respect to its first argument. Suppose that the local update oracle of \FedProx~is $\varepsilon_t$-inexactly solved and $\eta_t<\frac{1}{L}$. Then the following holds almost surely:
\[
\|\nabla \bar R_{\erm}(w_{t-1}) - \bar d_t\|^2  \le L^2 (G + 2L\varepsilon_t)^2\eta_t^2.
\]
\end{lemma}
\begin{proof}
By Lemma~\ref{lemma:inexact_key_2_fedprox} we know that
\begin{equation}\label{inequat:inexact_key_1_fedprox}
\|w_t^{(m)} - w_{t-1}\| \le \eta_t \|d^{(m)}_t\| + 2L\varepsilon_t \eta_t \le  (G+2L\varepsilon_t) \eta_t,
\end{equation}
where we have used the $G$-Lipschitz assumption of loss. By definition we can see that
\[
\begin{aligned}
\|\nabla \bar R_{\erm}(w_{t-1}) - \bar d_t\|^2 =&\left\|\frac{1}{M}\sum_{m=1}^M \left(\nabla R^{(m)}_{\erm}(w_{t-1}) -\nabla R^{(m)}_{\erm}(w^{(m)}_{t}) \right) \right\|^2 \\
 \le& \frac{1}{M}\sum_{m=1}^M\left\| \nabla R^{(m)}_{\erm}(w_{t-1}) -\nabla R^{(m)}_{\erm}(w^{(m)}_{t})\right\|^2 \\
 \overset{\zeta_1}{\le} & \frac{L^2}{M}\sum_{m=1}^M\left\| w_{t-1} - w^{(m)}_{t} \right\|^2 \\
 \overset{\zeta_2}{\le} & L^2 (G + 2L\varepsilon_t)^2\eta_t^2,
\end{aligned}
\]
where in ``$\zeta_1$'' we have used the $L$-smoothness of loss, in ``$\zeta_2$'' we have used~\eqref{inequat:inexact_key_1_fedprox}. This proves the desired bound.
\end{proof}

\newpage

With all the above lemmas in place, we can prove the main result in Theorem~\ref{thrm:fedprox_main_smooth}. Let $\{\mathcal{F}_{t}\}_{t\ge1}$ be the filtration generated by the random iterates $\{w_{t}\}_{t\ge1}$ as $\mathcal{F}_t = \sigma\left(w_1, w_2,..., w_t\right)$, where the randomness comes from the sampling of devices for partial participation.
\begin{proof}[Proof of Theorem~\ref{thrm:fedprox_main_smooth}]
Let us denote $\delta^{(m)}_t:= \eta_t^{-1}(w^{(m)}_t - w^{(t-1)}) + d^{(m)}_t$, $\delta_t:=\frac{1}{|I_t|} \sum_{\xi\in I_t} \delta^{(\xi)}_t$ and $\bar\delta_t:=\frac{1}{M} \sum_{m=1}^M \delta^{(m)}_t$. Then we have $\mathbb{E}[\delta_t] = \bar\delta_t$ and
\[
w_t = w_{t-1} -\eta_t (d_t - \delta_t).
\]
It can be verified based on Lemma~\ref{lemma:inexact_key_2_fedprox} and triangle inequality that the following holds almost surely:
\begin{equation}\label{proof:thrm_fedprox_key_1}
\max\left\{\|\bar \delta_t\|, \|\delta_t \|\right\} \le 2L \varepsilon_t.
\end{equation}
Since the loss is $L$-smooth, we can show that
\[
\begin{aligned}
&\mathbb{E}[\bar R_{\erm} (w_t)\mid \mathcal{F}_{t-1}] \\
\le& \mathbb{E}\left[\bar R_{\erm} (w_{t-1}) + \left\langle \nabla \bar R_{\erm}( w_{t-1}), w_{t} - w_{t-1}\right \rangle + \frac{L}{2}\|w_{t} - w_{t-1}\|^2\mid \mathcal{F}_{t-1} \right]\\
=& \mathbb{E}\left[\bar R_{\erm} (w_{t-1}) -\eta_t \left\langle \nabla\bar R_{\erm}(w_{t-1}), d_t - \delta_t \right\rangle + \frac{L\eta_t^2}{2}\left\|d_t - \delta_t\right\|^2\mid \mathcal{F}_{t-1} \right] \\
=& \bar R_{\erm} (w_{t-1}) +\mathbb{E}\left[-\eta_t \left\langle \nabla\bar R_{\erm}(w_{t-1}), \bar d_t - \bar \delta_t \right\rangle + \frac{L\eta_t^2}{2}\left\|d_t - \delta_t\right\|^2\mid \mathcal{F}_{t-1} \right] \\
\overset{\zeta_1}{\le}& \bar R_{\erm} (w_{t-1}) \\
&+ \mathbb{E}\left[-\eta_t \left\langle \nabla\bar R_{\erm}(w_{t-1}), \bar d_t \right\rangle + \eta_t G \|\bar \delta_t\| + \frac{3L\eta_t^2}{2}\|\bar d_t\|^2 + \frac{3L\eta_t^2}{2}\|d_t - \bar d_t\|^2 + \frac{3L\eta_t^2}{2}\|\delta_t\|^2\mid \mathcal{F}_{t-1} \right]\\
\overset{\zeta_2}{\le}& \bar R_{\erm} (w_{t-1}) + \mathbb{E}\left[- \frac{\eta_t}{2} \|\nabla\bar R_{\erm}(w_{t-1})\|^2 - \frac{\eta_t}{2} \|\bar d_t\|^2 + \frac{\eta_t}{2}\|\nabla \bar R_{\erm}(w_{t-1}) - \bar d_t\|^2+ 2LG\varepsilon_t\eta_t\right. \\
& \left. + \frac{3L\eta_t^2}{2}\left\|\bar d_t \right\|^2 + \frac{3LG^2\eta_t^2}{2 I} + 6\varepsilon_t^2 L^3\eta_t^2\mid \mathcal{F}_{t-1}\right] \\
\overset{\zeta_3}{\le} & \bar R_{\erm} (w_{t-1}) - \frac{\eta_t}{2} \|\nabla \bar R_{\erm}(w_{t-1})\|^2 + \mathbb{E}\left[\frac{\eta_t}{2}\|\nabla \bar R_{\erm}(w_{t-1}) - \bar d_t \|^2\mid \mathcal{F}_{t-1}\right] \\
&+ \frac{3LG^2\eta_t^2}{2 I}+ 2LG\varepsilon_t\eta_t + 6\varepsilon_t^2 L^3\eta_t^2  \\
\overset{\zeta_4}{\le} & \bar R_{\erm} (w_{t-1}) - \frac{\eta_t}{2} \|\nabla \bar R_{\erm}(w_{t-1})\|^2  + \frac{L^2 (G+2L\varepsilon_t)^2\eta_t^3}{2} + \frac{3LG^2\eta_t^2}{2 I} + 2LG\varepsilon_t\eta_t + 6\varepsilon_t^2 L^3\eta_t^2\\
\le& \bar R_{\erm} (w_{t-1}) - \frac{\eta_t}{2} \|\nabla \bar R_{\erm}(w_{t-1})\|^2  + 2L^2G^2\eta_t^3 + \frac{5LG^2\eta_t^2}{I},
\end{aligned}
\]
where in ``$\zeta_1$'' we have used the $G$-Lipschitz of loss and triangle inequality, in ``$\zeta_2$'' we have used Lemma~\ref{lemma:inexact_key_1} and \eqref{proof:thrm_fedprox_key_1}, in ``$\zeta_3$'' we have used $\eta_t\le \frac{1}{3L}$, in ``$\zeta_4$'' we have used the first bound of Lemma~\ref{lemma:fedprox_concentrain_key},  and in the last inequality we have used the condition $\varepsilon_t\le \min\left\{\frac{G}{2L\sqrt{I}}, \frac{G\eta_t}{I}\right\}$.

\newpage

Rearranging the terms and taking expectation over $\mathcal{F}_{t-1}$ in the above yields
\[
\mathbb{E}\left[\|\nabla \bar R_{\erm}(w_{t-1})\|^2\right\} \le \frac{2}{\eta_t} \mathbb{E}\left[\bar R_{\erm} (w_{t-1}) - \bar R_{\erm} (w_{t}) \right] + 4L^2G^2\eta_t^2 + \frac{10LG^2\eta_t}{I}.
\]
Averaging the above from over $t=1,2,..., T$ with $\eta_t\equiv \eta$ yields
\[
\begin{aligned}
\frac{1}{T}\sum_{t=0}^{T-1} \mathbb{E}\left[\|\nabla \bar R_{\erm}(w_{t})\|^2\right]  \le& \frac{2}{\eta T} \mathbb{E}\left[\bar R_{\erm} (w_0) - \bar R_{\erm}(w_T)\right] + 4L^2G^2\eta^2 + \frac{10LG^2\eta}{I} \\
\le&  \frac{2}{\eta T} \bar \Delta_{\erm} + 4L^2G^2\eta^2 + \frac{10LG^2\eta}{I}.
\end{aligned}
\]
If $T<I^3$, setting $\eta= \frac{1}{3LT^{1/3}}$ yields
\[
\frac{1}{T}\sum_{t=0}^{T-1} \mathbb{E} \left[ \|\nabla \bar R_{\erm}(w_{t})\|^2 \right] \lesssim \frac{L\bar \Delta_{\erm}+G^2}{T^{2/3}} + \frac{G^2}{T^{1/3}I} \lesssim \frac{L\bar \Delta_{\erm}+G^2}{T^{2/3}}.
\]
If $T\ge I^3$, setting $\eta= \frac{1}{3L}\sqrt{\frac{I}{T}}$ yields
\[
\frac{1}{T}\sum_{t=0}^{T-1} \mathbb{E} \left[ \|\nabla \bar R_{\erm}(w_{t})\|^2 \right] \lesssim \frac{L\bar \Delta_{\erm}+G^2}{\sqrt{TI}} + \frac{G^2I}{T} \lesssim \frac{L\bar \Delta_{\erm}+G^2}{\sqrt{TI}}.
\]
Combining the preceding two inequalities and  appealing to the definition of $w_{t^*}$ yields the desired bound.
\end{proof}

\subsection{Proof of Theorem~\ref{thrm:fedprox_main_nonsmooth}}
\label{ssect:fedprox_main_nonsmooth_proof}

We first present the following elementary lemma which will be used in the proof. It can be viewed as an inexact extension of the well-known three-point lemma to weakly convex functions.
\begin{lemma}\label{lemma:function_strong_convexity}
Let $f$ be a $\nu$-weakly convex function and $\eta< \frac{1}{\nu}$. Consider
\[
w^+=\argmin_{u} \left\{f(u) + \frac{1}{2\eta}\|u - w\|^2\right\}.
\]
Then for any $u$, we have
\[
f(w^+) + \frac{1}{2\eta}\|w^+ - w\|^2 \le f(u) + \frac{1}{2\eta}\|u - w\|^2 - \frac{1/\eta - \nu}{2}\|w^+ - u\|^2.
\]
\end{lemma}
\begin{proof}
Since $\eta < \frac{1}{\nu}$, we must have that the regularized objective $f(u) + \frac{1}{2\eta}\|u - w\|^2$ is $(1/\eta-\nu)$-strongly convex with respect to $u$, which immediately implies the desired bound.
\end{proof}

We will make use of the following lemma which shows that $w^{(m)}_t$ will be close to $w_{t-1}$ if the learning rate $\eta_t$ is small enough.

\newpage

\begin{lemma}\label{lemma:inexact_key_2_fedprox_nonsmooth}
Assume that for each $m\in [M]$, the loss function $\ell^{(m)}$ is $G$-Lipschitz and $\nu$-weakly convex with respect to its first argument. Suppose that the local update oracle of \FedProx~is exactly solved and $\eta_t<\frac{1}{\nu}$. Then it holds that
\[
\left\|w^{(m)}_t - w_{t-1} \right\| \le G\eta_t.
\]
\end{lemma}
\begin{proof}
Recall $Q^{(m)}_{\erm}(w;w_{t-1})=R^{(m)}_{\erm}(w) + \frac{1}{2\eta_t}\|w - w_{t-1}\|^2$. Since the loss function is $\nu$-weakly convex and $\eta_t<\frac{1}{\nu}$, $Q^{(m)}_{\erm}(w;w_{t-1})$ is strongly convex and thus admits a global minimizer. Since the local update oracle is exactly solved, we must have
\[
\left\|\nabla R^{(m)}_{\erm}(w^{(m)}_t) + \frac{1}{\eta_t}(w^{(m)}_t - w_{t-1})\right\| =0,
\]
which implies the desired bound due to the $G$-Lipschitzness.
\end{proof}

With the above two preliminary lemmas in place, we are now in the position to prove the main result in  Theorem~\ref{thrm:fedprox_main_nonsmooth}.

\vspace{0.1in}

\begin{proof}[Proof of Theorem~\ref{thrm:fedprox_main_nonsmooth}]
Since the losses are $\nu$-weakly convex and $\eta_t< \frac{1}{\nu}$, in view of Lemma~\ref{lemma:function_strong_convexity} we can show for each $m\in [M]$ that the following holds for any $w$,
\begin{equation}\label{proof:fedprox_nonsmooth_key_1}
R^{(m)}_{\erm} (w^{(m)}_t) + \frac{1}{2\eta_t} \|w^{(m)}_t - w_{t-1}\|^2 \le R^{(m)}_{\erm} (w) + \frac{1}{2\eta_t} \|w - w_{t-1}\|^2 - \frac{1/\eta_t - \nu}{2}\|w^{(m)}_t - w\|^2 .
\end{equation}
Let us denote
\[
\bar w_{t-1}:= \prox_{\rho \bar R_{\erm}}(w_{t-1}) = \argmin_{w} \left\{\bar R_{\erm} (w) + \frac{1}{2\rho}\|w- w_{t-1}\|^2\right\}.
\]
Setting $w=\bar w_{t-1}$ in the right hand side of~\eqref{proof:fedprox_nonsmooth_key_1} yields
\[
R^{(m)}_{\erm} (w^{(m)}_t) + \frac{1}{2\eta_t} \|w^{(m)}_t - w_{t-1}\|^2 \le R^{(m)}_{\erm} (\bar w_{t-1}) + \frac{1}{2\eta_t} \|\bar w_{t-1} - w_{t-1}\|^2 - \frac{1/\eta_t - \nu}{2}\|w^{(m)}_t - \bar w_{t-1}\|^2 .
\]
In view of the above inequality we can show that for any $\xi\in I_t$,
\begin{equation}\label{proof:fedprox_nonsmooth_key_2}
\begin{aligned}
&R^{(\xi)}_{\erm} (w_{t-1}) + \frac{1}{2\eta_t} \|w^{(\xi)}_t - w_{t-1}\|^2 \\
=& R^{(\xi)}_{\erm} (w^{(\xi)}_t) + \frac{1}{2\eta_t} \|w^{(\xi)}_t - w_{t-1}\|^2 + R^{(\xi)}_{\erm} (w_{t-1}) -  R^{(\xi)}_{\erm} (w^{(\xi)}_t)\\
\le& R^{(\xi)}_{\erm} (w^{(\xi)}_t) + \frac{1}{2\eta_t} \|w^{(\xi)}_t - w_{t-1}\|^2 + G\|w_{t-1} - w^{(\xi)}_t\| \\
\le& R^{(\xi)}_{\erm} (w^{(\xi)}_t) + \frac{1}{2\eta_t} \|w^{(\xi)}_t - w_{t-1}\|^2 + G^2\eta_t \\
\le& R^{(\xi)}_{\erm} (\bar w_{t-1}) + \frac{1}{2\eta_t} \|\bar w_{t-1} - w_{t-1}\|^2 - \frac{1/\eta_t - \nu}{2}\|w^{(\xi)}_t - \bar w_{t-1}\|^2 + G^2\eta_t,
\end{aligned}
\end{equation}
where in the last but one inequality we have applied Lemma~\ref{lemma:inexact_key_2_fedprox_nonsmooth}. 

Now recall that $w_t=\frac{1}{I}\sum_{\xi \in I_t} w^{(\xi)}_t$. Then based on triangle inequality we can see that
\[
\begin{aligned}
& \frac{1}{I} \sum_{\xi \in I_t}R^{(\xi)}_{\erm} (w_{t-1}) + \frac{1}{2\eta_t} \|w_t - w_{t-1}\|^2 \\
=&\frac{1}{I} \sum_{\xi \in I_t}R^{(\xi)}_{\erm} (w_{t-1}) + \frac{1}{2\eta_t} \left\|\frac{1}{I}\sum_{\xi \in I_t} w^{(\xi)}_t - w_{t-1}\right\|^2 \\
\le&  \frac{1}{I} \sum_{\xi \in I_t} \left\{R^{(\xi)}_{\erm} (w_{t-1}) + \frac{1}{2\eta_t} \left\| w^{(\xi)}_t - w_{t-1}\right\|^2\right\} \\
\overset{\zeta_1}{\le}& \frac{1}{I} \sum_{\xi \in I_t} \left\{R^{(\xi)}_{\erm} (\bar w_{t-1}) + \frac{1}{2\eta_t} \|\bar w_{t-1} - w_{t-1}\|^2 - \frac{1/\eta_t - \nu}{2}\|w^{(\xi)}_t - \bar w_{t-1}\|^2 + G^2\eta_t\right\}\\
\le&  \frac{1}{I} \sum_{\xi \in I_t} R^{(\xi)}_{\erm} (\bar w_{t-1}) + \frac{1}{2\eta_t} \|\bar w_{t-1} - w_{t-1}\|^2 - \frac{1/\eta_t - \nu}{2}\left\| \frac{1}{I} \sum_{\xi \in I_t} w^{(\xi)}_t - \bar w_{t-1}\right\|^2 + G^2\eta_t\\
=&\frac{1}{I} \sum_{\xi \in I_t} R^{(\xi)}_{\erm} (\bar w_{t-1}) + \frac{1}{2\eta_t} \|\bar w_{t-1} - w_{t-1}\|^2 - \frac{1/\eta_t - \nu}{2}\left\|w_t - \bar w_{t-1}\right\|^2 + G^2\eta_t,
\end{aligned}
\]
where in ``$\zeta_1$'' we have used \eqref{proof:fedprox_nonsmooth_key_2}. Conditioned on $\mathcal{F}_{t-1}$ , taking expectation over both sides of the above inequality leads to the following:
\[
\begin{aligned}
& \mathbb{E} \left[\bar R_{\erm} (w_{t-1}) + \frac{1}{2\eta_t} \|w_t - w_{t-1}\|^2  \mid  \mathcal{F}_{t-1}\right]\\
=& \mathbb{E} \left[ \frac{1}{I} \sum_{\xi \in I_t}R^{(\xi)}_{\erm} (w_{t-1}) + \frac{1}{2\eta_t} \|w_t - w_{t-1}\|^2\mid  \mathcal{F}_{t-1}\right] \\
\le& \mathbb{E} \left[ \frac{1}{I} \sum_{\xi \in I_t} R^{(\xi)}_{\erm} (\bar w_{t-1}) + \frac{1}{2\eta_t} \|\bar w_{t-1} - w_{t-1}\|^2 - \frac{1/\eta_t - \nu}{2}\left\|w_t - \bar w_{t-1}\right\|^2 + G^2\eta_t \mid  \mathcal{F}_{t-1}\right]\\
=& \mathbb{E} \left[ \bar R_{\erm} (\bar w_{t-1}) + \frac{1}{2\eta_t} \|\bar w_{t-1} - w_{t-1}\|^2 - \frac{1/\eta_t - \nu}{2}\left\|w_t - \bar w_{t-1}\right\|^2 + G^2\eta_t \mid  \mathcal{F}_{t-1}\right].
\end{aligned}
\]
Based the above inequality and by applying Lemma~\ref{lemma:inexact_key_2_fedprox_nonsmooth} again we can show that
\begin{equation}\label{proof:fedprox_nonsmooth_key_3}
\begin{aligned}
& \mathbb{E} \left[\bar R_{\erm} (w_t) + \frac{1}{2\eta_t} \|w_t - w_{t-1}\|^2  \mid  \mathcal{F}_{t-1}\right]\\
=& \mathbb{E} \left[\bar R_{\erm} (w_{t-1}) + \frac{1}{2\eta_t} \|w_t - w_{t-1}\|^2 + \bar R_{\erm} (w_{t}) - \bar R_{\erm} (w_{t-1}) \mid  \mathcal{F}_{t-1}\right]\\
\le & \mathbb{E} \left[\bar R_{\erm} (w_{t-1}) + \frac{1}{2\eta_t} \|w_t - w_{t-1}\|^2 + G\|w_{t} - w_{t-1}\| \mid  \mathcal{F}_{t-1}\right] \\
\le& \mathbb{E} \left[ \bar R_{\erm} (\bar w_{t-1}) + \frac{1}{2\eta_t} \|\bar w_{t-1} - w_{t-1}\|^2 - \frac{1/\eta_t - \nu}{2}\left\|w_t - \bar w_{t-1}\right\|^2 + 2G^2\eta_t \mid  \mathcal{F}_{t-1}\right],
\end{aligned}
\end{equation}
where in the last inequality we have used $\|w_t - w_{t-1}\|\le \frac{1}{I}\sum_{\xi \in I_t}\|w^{(\xi)}_t - w_{t-1}\| \le G\eta_t$ due to triangle inequality and Lemma~\ref{lemma:inexact_key_2_fedprox_nonsmooth}.

Since $\bar R_{\erm}$ is also $\nu$-weakly convex, invoking Lemma~\ref{lemma:function_strong_convexity} to $\bar w_{t-1} = \prox_{\rho \bar R_{\erm}}(w_{t-1})$ yields
\[
\bar R_{\erm} (\bar w_{t-1}) + \frac{1}{2\rho} \|\bar w_{t-1} - w_{t-1}\|^2 \le \bar R_{\erm} (w_t) + \frac{1}{2\rho} \|w_t - w_{t-1}\|^2 - \frac{1/\rho - \nu}{2}\|\bar w_{t-1} - w_t\|^2,
\]
which immediately gives the following conditioned expectation bound:
\begin{equation}\label{proof:fedprox_nonsmooth_key_4}
\begin{aligned}
&\mathbb{E}\left[\bar R_{\erm} (\bar w_{t-1}) + \frac{1}{2\rho} \|\bar w_{t-1} - w_{t-1}\|^2\mid \mathcal{F}_{t-1}\right] \\
\le& \mathbb{E}\left[\bar R_{\erm} (w_t) + \frac{1}{2\rho} \|w_t - w_{t-1}\|^2 - \frac{1/\rho - \nu}{2}\|\bar w_{t-1} - w_t\|^2 \mid \mathcal{F}_{t-1}\right].
\end{aligned}
\end{equation}
By summing up~\eqref{proof:fedprox_nonsmooth_key_3} and~\eqref{proof:fedprox_nonsmooth_key_4} we have
\[
\begin{aligned}
&\mathbb{E} \left[ \frac{1/\eta_t - 1/\rho}{2} \|w_t - w_{t-1}\|^2  \mid  \mathcal{F}_{t-1}\right] \\
\le& \mathbb{E} \left[\frac{1/\eta_t - 1/\rho}{2} \|\bar w_{t-1} - w_{t-1}\|^2 -\frac{1/\eta_t + 1/\rho -2\nu}{2} \|\bar w_{t-1} - w_t\|^2 + 2G^2\eta_t\mid  \mathcal{F}_{t-1}\right].
\end{aligned}
\]
Since by assumption $\eta_t\le \rho$, rearranging the terms in the above yields
\[
\begin{aligned}
&\mathbb{E} \left[ \|w_t - \bar w_{t-1}\|^2 \mid \mathcal{F}_{t-1}\right] \\
\le& \frac{1/\eta_t - 1/\rho}{1/\eta_t + 1/\rho - 2\nu}\|\bar w_{t-1} - w_{t-1}\|^2 + \frac{4G^2\eta_t}{1/\eta_t + 1/\rho - 2\nu } \\
\le& \|\bar w_{t-1} - w_{t-1}\|^2 - \frac{2(1/\rho - \nu)}{1/\eta_t + 1/\rho - 2\nu} \|\bar w_{t-1} - w_{t-1}\|^2 + \frac{4G^2\eta_t}{1/\eta_t + 1/\rho - 2\nu }.
\end{aligned}
\]
Then based on the above and the definition of Moreau envelope we can show that
\[
\begin{aligned}
&\mathbb{E} \left[ \bar R_{\erm, \rho}(w_t) \mid \mathcal{F}_{t-1}\right] \\
=& \mathbb{E} \left[ \bar R_{\erm}(\bar  w_t) + \frac{1}{2\rho} \|\bar w_t - w_t\|^2\mid \mathcal{F}_{t-1}\right] \\
\le& \mathbb{E} \left[ \bar R_{\erm}(\bar  w_{t-1}) + \frac{1}{2\rho} \|\bar w_{t-1} - w_t\|^2\mid \mathcal{F}_{t-1}\right] \\
=& \bar R_{\erm}(\bar  w_{t-1}) + \frac{1}{2\rho} \mathbb{E} \left[ \|\bar w_{t-1} - w_t\|^2\mid \mathcal{F}_{t-1}\right]\\
\le& \bar R_{\erm}(\bar  w_{t-1}) + \frac{1}{2\rho}\|\bar w_{t-1} - w_{t-1}\|^2 - \frac{(1/\rho - \nu)/\rho}{1/\eta_t + 1/\rho - 2\nu} \|\bar w_{t-1} - w_{t-1}\|^2 + \frac{2G^2\eta_t/\rho}{1/\eta_t + 1/\rho - 2\nu }\\
=& \bar R_{\erm,\rho}(w_{t-1}) - \frac{(1/\rho - \nu)/\rho}{1/\eta_t + 1/\rho - 2\nu} \|\bar w_{t-1} - w_{t-1}\|^2 + \frac{2G^2\eta_t/\rho}{1/\eta_t + 1/\rho - 2\nu } \\
=& \bar R_{\erm,\rho}(w_{t-1}) - \frac{1 - \rho\nu}{1/\eta_t + 1/\rho - 2\nu} \left\|\nabla \bar R_{\erm, \rho}(w_{t-1})\right\|^2 + \frac{2G^2\eta_t/\rho}{1/\eta_t + 1/\rho - 2\nu } ,
\end{aligned}
\]
where in the last equality we have used the identity $\|\bar w_{t-1} - w_{t-1}\|^2 = \rho^2 \left\|\nabla \bar R_{\erm, \rho}(w_{t-1})\right\|^2$~\citep[see, e.g., ][]{davis2019stochastic}. 

\newpage

By rearranging the terms in the above and taking expectation over $\mathcal{F}_{t-1}$ we obtain that
\[
\frac{1 - \rho\nu}{1/\eta_t + 1/\rho - 2\nu} \mathbb{E}\left[\left\|\nabla \bar R_{\erm, \rho}(w_{t-1})\right\|^2\right] \le \mathbb{E}\left[\bar R_{\erm,\rho}(w_{t-1})\right] -\mathbb{E}\left[\bar R_{\erm,\rho}(w_t)\right]+ \frac{2G^2\eta_t/\rho}{1/\eta_t + 1/\rho - 2\nu }.
\]
Averaging the above over $t=1,...,T$ yields
\[
\begin{aligned}
\frac{1}{T} \sum_{t=0}^{T-1}  \mathbb{E}\left[\left\|\nabla \bar R_{\erm, \rho}(w_{t})\right\|^2\right] \le& \frac{1/\eta_t + 1/\rho - 2\nu}{T(1-\rho\nu)} \mathbb{E}\left[\bar R_{\erm,\rho}(w_{0}) - \bar R_{\erm,\rho}(w_{T})\right] + \frac{2G^2\eta_t}{\rho(1-\rho\nu)} \\
\le& \frac{1/\eta_t + 1/\rho - 2\nu}{T(1-\rho\nu)} \bar \Delta_{\erm,\rho} + \frac{2G^2\eta_t}{\rho(1-\rho\nu)} \\
=& \frac{(1 - 2\rho \nu)\bar \Delta_{\erm,\rho}}{T\rho (1-\rho\nu)} + \frac{\bar \Delta_{\erm,\rho}}{\eta_t T(1-\rho\nu)} + \frac{2G^2\eta_t}{\rho(1-\rho\nu)}\\
\le& \frac{\bar \Delta_{\erm,\rho}}{T\rho} + \frac{2\bar \Delta_{\erm,\rho}}{\eta_t T} + \frac{4G^2\eta_t}{\rho} \\
=& \frac{\bar \Delta_{\erm,\rho}}{T\rho} + \frac{2\bar \Delta_{\erm,\rho}+ 4G^2\rho}{\rho \sqrt{T}},
\end{aligned}
\]
where in the last but one inequality we have used $\rho<\frac{1}{2\nu}$, and in the last inequality we have used the choice of $\eta_t\equiv \frac{\rho}{\sqrt{T}}$. The desired bound follows by preserving the dominant terms in the above bound and appealing to the definition of $t^*$.
\end{proof}

\section{Proofs for Section~\ref{sect:analysis_fedmspp}}
\label{apdx:proofs_fedmspp}

\subsection{Proof of Theorem~\ref{thrm:fedmspp_main_smooth}}
\label{ssect:proof_fedmspp_main_smooth}

For each time instance $t$, let us overload the notation $d^{(m)}_t$ as
\[
d^{(m)}_t = \nabla R^{(m)}_{B^{(m)}_t}(w^{(m)}_t) = \frac{1}{b}\sum_{i=1}^b \nabla\ell^{(m)}(w^{(m)}_t; z^{(m)}_{i,t}).
\]
We then accordingly overload the quantities $d_t$ and $\bar d_t$ as defined in~\eqref{equat:d_t_bar}. We then have the following lemma analogous to Lemma~\ref{lemma:inexact_key_2_fedprox}.
\begin{lemma}\label{lemma:inexact_key_2_fedmspp}
Assume that for each $m\in [M]$, the loss function $\ell^{(m)}$ is $L$-smooth with respect to its first argument. Suppose that the local update oracle of \FedMSPP~is $\varepsilon_t$-inexactly solved and $\eta_t<\frac{1}{L}$. Then it holds that
\[
\left\|w^{(m)}_t - w_{t-1} + \eta_t d^{(m)}_t \right\| \le 2L\varepsilon_t\eta_t.
\]
\end{lemma}
\begin{proof}
Consider $Q^{(m)}_{B^{(m)}_t}(w; w_{t-1})=R^{(m)}_{B^{(m)}_t}(w) + \frac{1}{2\eta_t}\|w - w_{t-1}\|^2$. Since the loss functions are $L$-smooth and $\eta_t<\frac{1}{L}$, we know that $Q^{(m)}_{B^{(m)}_t}(w;w_{t-1})$ must be strongly convex and thus admits a global minimizer. Then we have
\[
\begin{aligned}
&\left\|\nabla R^{(m)}_{B^{(m)}_t}(w^{(m)}_t) + \frac{1}{\eta_t}(w^{(m)}_t - w_{t-1})\right\| \\
=& \left\|\nabla Q^{(m)}_{B^{(m)}_t}(w^{(m)}_t; w_{t-1})\right\|\le 2L \left(Q^{(m)}_{B^{(m)}_t}(w^{(m)}_t; w_{t-1}) - \min_w Q^{(m)}_{B^{(m)}_t}(w; w_{t-1})\right) \le 2L\varepsilon_t,
\end{aligned}
\]
where in the last inequality is due to Definition~\ref{def:inexact_oracle_fedmspp}. This implies the desired bound.
\end{proof}

Let $\{\mathcal{F}_{t}\}_{t\ge1}$ be the filtration generated by the random iterates $\{w_{t}\}_{t\ge1}$ as $\mathcal{F}_t = \sigma\left(w_1, w_2,..., w_t\right)$, where the randomness jointly comes from the sampling of devices for partial participation and sampling of minibatch for local update on each chosen device.

\begin{lemma}\label{lemma:minibatch_erm_momentbound}
Assume that for each $m\in [M]$, the loss function $\ell^{(m)}$ is $G$-Lipschitz and $L$-smooth with respect to its first argument. Suppose that $\eta_t<\frac{1}{L}$ and the local update oracle of \FedMSPP~is $\varepsilon_t$-inexactly solved with $\varepsilon_t \le \frac{G^2\eta_t}{8b^2}$. Then it holds for every $m\in [M]$ that
\[
\begin{aligned}
\left\|\mathbb{E}\left[\nabla R^{(m)}(w^{(m)}_{t})- d^{(m)}_t \mid \mathcal{F}_{t-1}\right]\right\|\le& \frac{5LG\eta_t}{(1-\eta_tL)b}, \\
%\mathbb{E}\left[\left\|\nabla R^{(m)}(w^{(m)}_{t})- d^{(m)}_t \right\|^2 \mid \mathcal{F}_{t-1}\right] \le& \frac{850L^2G^2}{(1/\eta_t-L)^2b^2}+\frac{8G^2}{b},\\
\mathbb{E}\left[\left\|\nabla R^{(m)}(w^{(m)}_{t}) - \mathbb{E}[ \nabla R^{(m)}(w^{(m)}_{t}) \mid \mathcal{F}_{t-1}]\right\|^2 \mid \mathcal{F}_{t-1}\right] \le& \frac{25L^2G^2\eta_t}{(1-\eta_tL)^2b}.
\end{aligned}
\]
\end{lemma}
\begin{proof}
Let us recall Definition~\ref{def:inexact_oracle_fedmspp} where the inexact solution $w^{(m)}_t$ is given by
\[
 Q^{(m)}_{B^{(m)}_t}(w^{(m)}_t; w_{t-1}) \le  \min_{w} Q^{(m)}_{B^{(m)}_t}(w; w_{t-1}) + \varepsilon_t.
\]
Since the loss functions are $L$-smooth and $\frac{1}{\eta_t} > L$, it is easy to verify that the regularized objective $Q^{(m)}_{B^{(m)}_t}(w; w_{t-1})$ is $(\frac{1}{\eta_t}-L)$-strongly convex. Then invoking Lemma~\ref{lemma:stability_rERM} yields that $w^{(m)}_t$ uniformly stable with parameter $\frac{4G}{(1/\eta_t-L)b} + 2\sqrt{\frac{2\varepsilon_t}{1/\eta_t-L}}\le \frac{5G}{(1/\eta_t-L)b}$, which is due to the condition on $\varepsilon_t$. Conditioned on the sigma-field $\mathcal{F}_{t-1}$, the desired bounds follows immediately from Lemma~\ref{lemma:stability_genalization_first_order_moment}.
\end{proof}
The next lemma, which can be proved based on the previous lemmas, is key to our analysis.
\begin{lemma}\label{lemma:concentrain_key}
Assume that for each $m\in [M]$, the loss function $\ell^{(m)}$ is $G$-Lipschitz and $L$-smooth with respect to its first argument. Suppose that $\eta_t<\frac{1}{L}$ and the local update oracle of \FedMSPP~is $\varepsilon_t$-inexactly solved with $\varepsilon_t \le \min\left\{\frac{G}{2L}, \frac{G^2\eta_t}{8b^2}\right\}$. Then we have
\[
\mathbb{E}\left[\left\|\nabla \bar R(w_{t-1}) - d_t\right\|^2 \mid \mathcal{F}_{t-1} \right] \le 8L^2G^2\eta_t^2 + \frac{2G^2}{b|I_t|},
\]
and
\[
\left\|\nabla \bar R(w_{t-1}) - \mathbb{E}[\bar d_t\mid \mathcal{F}_{t-1}]\right\|^2 \le 12L^2G^2\eta_t^2 + \frac{75L^2G^2\eta^2_t}{(1-\eta_tL)^2b^2} + \frac{75L^2G^2\eta_t^2}{(1-\eta_tL)^2b}.
\]
\end{lemma}
\begin{proof}
By Lemma~\ref{lemma:inexact_key_2_fedmspp} we know that for each $m\in [M]$,
\begin{equation}\label{inequat:inexact_key_1_fedmspp}
\|w_t^{(m)} - w_{t-1}\| \le \eta_t \|d^{(m)}_t\| + 2L \varepsilon_t \eta_t \le (G+2L\varepsilon_t)\eta_t\le 2G\eta_t,
\end{equation}
where we have used the $G$-Lipschitz assumption of loss and $\varepsilon_t\le\frac{G}{2L}$.

By definition we can see that
\[
\begin{aligned}
&\mathbb{E}\left[|\nabla \bar R(w_{t-1}) - d_t\|^2 \mid \mathcal{F}_{t-1} \right]\\
=&\mathbb{E}\left[\left\|\nabla \bar R(w_{t-1}) - \frac{1}{|I_t|}\sum_{\xi \in I_t} \nabla R^{(\xi)}_{B^{(\xi)}_t}(w^{(\xi)}_t) \right\|^2 \mid \mathcal{F}_{t-1} \right] \\
=& \mathbb{E}\left[\left\|\nabla \bar R(w_{t-1}) -\frac{1}{|I_t|}\sum_{\xi \in I_t} \nabla R^{(\xi)}_{B^{(\xi)}_t}(w_{t-1}) + \frac{1}{|I_t|}\sum_{\xi \in I_t} \nabla R^{(\xi)}_{B^{(\xi)}_t}(w_{t-1}) - \frac{1}{|I_t|}\sum_{\xi \in I_t} \nabla R^{(\xi)}_{B^{(\xi)}_t}(w^{(\xi)}_t) \right\|^2 \mid \mathcal{F}_{t-1} \right] \\
\le& \mathbb{E}\left[2\left\|\nabla \bar R(w_{t-1}) -\frac{1}{b|I_t|}\sum_{\xi \in I_t}\sum_{i\in [b]} \nabla\ell^{(\xi)}(w_{t-1}; z^{(\xi)}_{i,t})\right\|^2 + \frac{2}{|I_t|}\sum_{\xi\in I_t}\left\|\nabla R^{(\xi)}_{B^{(\xi)}_t}(w_{t-1}) - \nabla R^{(\xi)}_{B^{(\xi)}_t}(w^{(\xi)}_{t}) \right\|^2 \mid \mathcal{F}_{t-1} \right]\\
 \overset{\zeta_1}{\le}& \frac{2}{b^2|I_t|^2}\sum_{\xi \in I_t}\sum_{i\in [b]} \mathbb{E}\left[\left\|\nabla \bar R(w_{t-1}) - \nabla\ell^{(\xi)}(w_{t-1}; z^{(\xi)}_{i,t})\right\|^2 \mid \mathcal{F}_{t-1} \right] + \frac{2L^2}{|I_t|}\sum_{\xi\in I_t}\mathbb{E}\left[\left\|w_{t-1} - w^{(\xi)}_{t} \right\|^2 \mid \mathcal{F}_{t-1} \right] \\
  \overset{\zeta_2}{\le} & \frac{2}{b^2|I_t|^2}\sum_{\xi \in I_t}\sum_{i\in [b]} \mathbb{E}\left[\left\|\nabla\ell^{(\xi)}(w_{t-1}; z^{(\xi)}_{i,t})\right\|^2 \mid \mathcal{F}_{t-1} \right] + 8L^2G^2\eta_t^2\\
  \le& \frac{2G^2}{b|I_t|} + 8L^2 G^2\eta_t^2,
\end{aligned}
\]
where in ``$\zeta_1$'' we have used the independent sampling of data and devices and the $L$-smoothness of loss, in ``$\zeta_2$'' we have used the fact $\mathbb{E}[\|Z - \mathbb{E}[Z]\|^2]\le \mathbb{E}[\|Z\|^2]$ and~\eqref{inequat:inexact_key_1_fedmspp}, and in the last inequality we have used the $G$-Lipschitzness of loss . This proves the first desired bound.

To prove the second bound, by definition we can see that
\[
\begin{aligned}
&\left\|\nabla \bar R(w_{t-1}) - \mathbb{E}[\bar d_t\mid \mathcal{F}_{t-1}] \right\|^2 \\
=& \left\|\frac{1}{M}\sum_{m=1}^M \left(\nabla R^{(m)}( w_{t-1}) - \mathbb{E}[d^{(m)}_t\mid \mathcal{F}_{t-1}] \right) \right\|^2 \\
 =&\left\|\frac{1}{M}\sum_{m=1}^M \left(\nabla R^{(m)}( w_{t-1}) -\nabla R^{(m)}(w^{(m)}_{t}) + \nabla R^{(m)}(w^{(m)}_{t}) - \mathbb{E}[\nabla R^{(m)}(w^{(m)}_{t})\mid \mathcal{F}_{t-1}] \right. \right. \\
 & \left.\left. \quad\quad\quad\quad\quad\quad\quad\quad\quad\quad\quad\quad\quad\quad\quad\quad + \mathbb{E}\left[\nabla R^{(m)}(w^{(m)}_{t})-d^{(m)}_t\mid \mathcal{F}_{t-1}\right] \right) \right\|^2 \\
\le& \underbrace{\frac{3}{M}\sum_{m=1}^M\left\| \nabla R^{(m)}(w_{t-1}) -\nabla R^{(m)}(w^{(m)}_{t})\right\|^2}_{A'} +  \underbrace{\frac{3}{M}\sum_{m=1}^M  \left\| \mathbb{E}\left[\nabla R^{(m)}(w^{(m)}_{t})- d^{(m)}_t \mid \mathcal{F}_{t-1} \right] \right\|^2 }_{B'} \\
 &+\underbrace{\frac{3}{M}\sum_{m=1}^M \left\|\nabla R^{(m)}(w^{(m)}_{t}) - \mathbb{E}[\nabla R^{(m)}(w^{(m)}_{t})\mid \mathcal{F}_{t-1}]\right\|^2}_{C'}.
\end{aligned}.
\]

\newpage

By smoothness and~\eqref{inequat:inexact_key_1_fedmspp} we can show that the following holds almost surely:
\[
A' \le 3L^2(G+2L\varepsilon_t)^2\eta_t^2 \le 12L^2G^2\eta^2_t.
\]
For the component $B'$, based on the first bound of Lemma~\ref{lemma:minibatch_erm_momentbound} we can easily show that
\[
B' \le \frac{75 L^2G^2}{(1/\eta_t-L)^2b^2}.
\]
In terms of the component $C'$, it can be bounded via invoking the second bound of Lemma~\ref{lemma:minibatch_erm_momentbound} that
\[
\mathbb{E}\left[C' \mid \mathcal{F}_{t-1}\right] \le \frac{75L^2G^2}{(1/\eta_t-L)^2b}.
\]
Finally, by combing the preceding three bounds we obtain that
\[
\begin{aligned}
\left\|\nabla \bar R(w_{t-1}) - \mathbb{E}[\bar d_t\mid \mathcal{F}_{t-1}]\right\|^2  =& \mathbb{E}\left[\left\|\nabla \bar R(w_{t-1}) - \mathbb{E}[\bar d_t\mid \mathcal{F}_{t-1}]\right\|^2 \mid \mathcal{F}_{t-1}\right] \\
\le& 12L^2G^2\eta_t^2 + \frac{75L^2G^2}{(1/\eta_t-L)^2b^2} + \frac{75L^2G^2}{(1/\eta_t-L)^2b} .
\end{aligned}
\]
This proves the second desired bound.
\end{proof}

\vspace{0.2in}

With all the above preliminary results in place, we are now ready to prove the main result in Theorem~\ref{thrm:fedmspp_main_smooth}.
\begin{proof}[Proof of Theorem~\ref{thrm:fedmspp_main_smooth}]
Let us denote $\delta^{(m)}_t:= \eta_t^{-1}(w^{(m)}_t - w^{(t-1)}) + d^{(m)}_t$, $\delta_t:=\frac{1}{|I_t|} \sum_{\xi\in I_t} \delta^{(\xi)}_t$ and $\bar\delta_t:=\frac{1}{M} \sum_{m=1}^M \delta^{(m)}_t$. Then we have $\mathbb{E}[\delta_t] = \bar\delta_t$ and
\[
w_t = w_{t-1} -\eta_t (d_t - \delta_t).
\]
It can be verified based on Lemma~\ref{lemma:inexact_key_2_fedprox} and triangle inequality that the following holds almost surely:
\begin{equation}\label{proof:thrm_fedmspp_key_1}
\max\left\{\|\bar \delta_t\|, \|\delta_t \|\right\} \le 2L \varepsilon_t.
\end{equation}
Since the objective is $L$-smooth, we can show that
\[
\begin{aligned}
&\mathbb{E} \left[\bar R (w_t) \mid \mathcal{F}_{t-1}\right] \\
\le& \mathbb{E} \left[\bar R (w_{t-1}) + \left\langle \nabla \bar R(w_{t-1}), w_{t} - w_{t-1}\right \rangle + \frac{L}{2}\|w_{t} - w_{t-1}\|^2\mid \mathcal{F}_{t-1}\right] \\
=& \mathbb{E}\left[\bar R(w_{t-1}) -\eta_t \left\langle \nabla\bar R(w_{t-1}), d_t - \delta_t \right\rangle + \frac{L\eta_t^2}{2}\left\|d_t - \delta_t\right\|^2\mid \mathcal{F}_{t-1} \right] \\
=& \mathbb{E} \left[\bar R (w_{t-1}) -\eta_t \left\langle \nabla \bar R(w_{t-1}), \mathbb{E}[\bar d_t - \bar\delta_t \mid \mathcal{F}_{t-1}]\right \rangle + \frac{L\eta_t^2}{2}\| d_t - \delta_t\|^2\mid \mathcal{F}_{t-1}\right] \\
\le& \bar R (w_{t-1}) -\eta_t \left\langle \nabla\bar R(w_{t-1}), \mathbb{E}[\bar d_t\mid \mathcal{F}_{t-1}] \right\rangle + \mathbb{E}\left[\eta_t G \|\bar \delta_t\| + L\eta_t^2\|d_t\|^2 + L\eta_t^2\|\delta_t\|^2\mid \mathcal{F}_{t-1} \right]\\
=& \bar R (w_{t-1}) - \frac{\eta_t}{2} \left\|\nabla\bar R(w_{t-1})\right\|^2 - \frac{\eta_t}{2} \left\|\mathbb{E}[\bar d_t\mid \mathcal{F}_{t-1}]\right\|^2 + \frac{\eta_t}{2}\left\|\nabla \bar R(w_{t-1}) - \mathbb{E}[\bar d_t\mid \mathcal{F}_{t-1}]\right\|^2 \\
& + \mathbb{E}\left[\eta_t G \|\bar \delta_t\| + L\eta_t^2\|d_t\|^2 + L\eta_t^2\|\delta_t\|^2\mid \mathcal{F}_{t-1} \right] \\
\overset{\zeta_1}{\le}& \bar R (w_{t-1}) - \frac{\eta_t}{2} \left\|\nabla\bar R(w_{t-1})\right\|^2 + \frac{\eta_t}{2}\left\|\nabla \bar R(w_{t-1}) - \mathbb{E}[\bar d_t\mid \mathcal{F}_{t-1}]\right\|^2 \\
& + \mathbb{E}\left[ 2LG\varepsilon_t\eta_t + L\eta_t^2\left\| d_t \right\|^2  + 4 L^3\eta_t^2\varepsilon_t^2\mid \mathcal{F}_{t-1}\right] \\
\le&\bar R (w_{t-1}) - \frac{\eta_t}{2} \|\nabla\bar R(w_{t-1})\|^2 + \frac{\eta_t}{2}\|\nabla \bar R(w_{t-1}) - \mathbb{E}[\bar d_t\mid \mathcal{F}_{t-1}]\|^2 \\
& + \mathbb{E} \left[ 2L\eta_t^2 \left\|d_t - \nabla\bar R(w_{t-1}) \right\|^2 + 2L\eta_t^2\left\|\nabla\bar R(w_{t-1}) \right\|^2 +  2LG\varepsilon_t\eta_t + 4 L^3\eta_t^2\varepsilon_t^2\mid \mathcal{F}_{t-1}\right] \\
\overset{\zeta_2}{\le}& \bar R (w_{t-1}) - \frac{\eta_t}{4} \left\|\nabla \bar R(w_{t-1})\right\|^2 + \frac{\eta_t}{2}\left\|\nabla \bar R(w_{t-1}) - \mathbb{E}[\bar d_t\mid \mathcal{F}_{t-1}]\right\|^2  \\
& + \mathbb{E} \left[ 2L\eta_t^2 \left\|d_t - \nabla\bar R(w_{t-1}) \right\|^2 \mid \mathcal{F}_{t-1}\right]+ 2LG\varepsilon_t\eta_t +4 L^3\eta_t^2\varepsilon_t^2 \\
\overset{\zeta_3}{\le} &\bar R (w_{t-1}) - \frac{\eta_t}{4} \left\|\nabla \bar R(w_{t-1})\right\|^2 + 6L^2G^2\eta_t^3 + \frac{38 L^2G^2\eta^3_t}{(1-\eta_tL)^2b^2} + \frac{38L^2G^2\eta^3_t}{(1 - \eta_tL)^2b} \\
&+ 16L^3G^2\eta_t^4 + \frac{4LG^2\eta^2_t}{bI} +  2LG\eta_t\varepsilon_t +4 L^3\eta_t^2\varepsilon_t^2\\
\overset{\zeta_4}{\le} & \bar R (w_{t-1}) - \frac{\eta_t}{4} \left\|\nabla \bar R(w_{t-1})\right\|^2 + 6L^2G^2\eta_t^3 + \frac{43 L^2G^2\eta^3_t}{b^2} + \frac{43L^2G^2\eta^3_t}{b} \\
&+ 16L^3G^2\eta_t^4 + \frac{4LG^2\eta^2_t}{bI} +  2LG\eta_t\varepsilon_t +4 L^3\eta_t^2\varepsilon_t^2\\
\overset{\zeta_5}{\le} & \bar R (w_{t-1}) - \frac{\eta_t}{4} \left\|\nabla \bar R(w_{t-1})\right\|^2  + 94L^2G^2\eta_t^3 + \frac{4LG^2\eta^2_t}{bI} +  2LG\eta_t\varepsilon_t +4 L^3\eta_t^2\varepsilon_t^2\\
\le & \bar R (w_{t-1}) - \frac{\eta_t}{4} \left\|\nabla \bar R(w_{t-1})\right\|^2  + 94L^2G^2\eta_t^3 + \frac{6LG^2\eta^2_t}{bI},
\end{aligned}
\]
where in ``$\zeta_1$'' we have used~\eqref{proof:thrm_fedmspp_key_1}, in ``$\zeta_2$'' we have used $\eta_t\le \frac{1}{8L}$, in ``$\zeta_3$'' we have used Lemma~\ref{lemma:concentrain_key}, in ``$\zeta_4$'' we have used $\eta_t\le \frac{1}{8L}$, in ``$\zeta_5$'' we have used $M,b \ge 1$ and $\eta_t\le \frac{1}{8L}$, and in the last inequality we used $\varepsilon_t\le \frac{G\eta_t}{2bI}$ and $\eta_t\le \frac{1}{8L}$. By taking expectation over $\mathcal{F}_{t-1}$ and rearranging the terms we obtain that
\[
\begin{aligned}
\mathbb{E} \left[ \left\|\nabla \bar R(w_{t-1})\right\|^2 \right] \le& \frac{4}{\eta_t} \left(\mathbb{E}[\bar R (w_{t-1}) ] - \mathbb{E}[\bar R (w_{t}) ] \right)+ 376 L^2G^2\eta_t^2 + \frac{24LG^2\eta_t}{bI} .
\end{aligned}
\]
Averaging the above from over $t=1,2,.., T$ with $\eta_t\equiv \eta$ yields
\[
\begin{aligned}
\frac{1}{T}\sum_{t=0}^{T-1} \mathbb{E} \left[ \left\|\nabla \bar R(w_{t})\right\|^2 \right] \le& \frac{4}{\eta T} (\bar R (w_0) - \bar R(w_T)) +  376L^2G^2\eta^2 + \frac{24L\eta G^2}{bI}\\
\le& \frac{4\bar \Delta}{\eta T}  +  376L^2G^2\eta^2 + \frac{24L\eta G^2}{bI}.
\end{aligned}
\]
If $T<(bI)^3$, setting $\eta= \frac{1}{8LT^{1/3}}$ yields
\[
\frac{1}{T}\sum_{t=0}^{T-1} \mathbb{E} \left[ \|\nabla \bar R(w_{t})\|^2 \right] \lesssim \frac{L\bar \Delta+G^2}{T^{2/3}} + \frac{G^2}{T^{1/3}bI} \lesssim \frac{L\bar \Delta+G^2}{T^{2/3}}.
\]
If $T\ge (bI)^3$, setting $\eta= \frac{1}{8L}\sqrt{\frac{bI}{T}}$ yields
\[
\frac{1}{T}\sum_{t=0}^{T-1} \mathbb{E} \left[ \|\nabla \bar R(w_{t})\|^2 \right] \lesssim \frac{L\bar \Delta+G^2}{\sqrt{TbI}} + \frac{G^2bI}{T} \lesssim \frac{L\bar \Delta + G^2}{\sqrt{TbI}}.
\]
Combining the preceding two inequalities yields the desired bound.
\end{proof}

\subsection{Proof of Theorem~\ref{thrm:fedmspp_main_nonsmooth}}
\label{ssect:fedmspp_main_nonsmooth_proof}

The proof argument is almost identical to that of Theorem~\ref{thrm:fedprox_main_nonsmooth}. We reproduce the proof in full details here for the sake of completeness.

Similar to Lemma~\ref{lemma:inexact_key_2_fedprox_nonsmooth}, we first establish the following lemma which shows that $w^{(m)}_t$ will be close to $w_{t-1}$ if the learning rate $\eta_t$ is small enough.
\begin{lemma}\label{lemma:inexact_key_2_fedmspp_nonsmooth}
Assume that for each $m\in [M]$, the loss function $\ell^{(m)}$ is $G$-Lipschitz and $\nu$-weakly convex with respect to its first argument. Suppose that the local update oracle of \FedMSPP~is exactly solved and $\eta_t<\frac{1}{\nu}$. Then it holds that
\[
\left\|w^{(m)}_t - w_{t-1} \right\| \le G\eta_t.
\]
\end{lemma}
\begin{proof}
Recall $Q^{(m)}_{B^{(m)}_t}(w; w_{t-1})=R^{(m)}_{B^{(m)}_t}(w) + \frac{1}{2\eta_t}\|w - w_{t-1}\|^2$. Since the loss function is $\nu$-weakly convex and $\eta_t<\frac{1}{\nu}$, $Q^{(m)}_{B^{(m)}_t}(w; w_{t-1})$ is strongly convex with respect to $w$ and thus admits a global minimizer. Since the local update oracle is exactly solved, we must have
\[
\left\|\nabla R^{(m)}_{B^{(m)}_t}(w^{(m)}_t) + \frac{1}{\eta_t}(w^{(m)}_t - w_{t-1})\right\| =0,
\]
which implies the desired bound due to the $G$-Lipschitz-loss assumption.
\end{proof}

We are now ready to prove the main result in  Theorem~\ref{thrm:fedmspp_main_nonsmooth}.

\begin{proof}[Proof of Theorem~\ref{thrm:fedmspp_main_nonsmooth}]
Since the losses are $\nu$-weakly convex and $\eta_t< \frac{1}{\nu}$, in view of Lemma~\ref{lemma:function_strong_convexity} we can show for each $m\in [M]$ that the following holds for any $w$,
\begin{equation}\label{proof:fedmspp_nonsmooth_key_1}
R^{(m)}_{B^{(m)}_t} (w^{(m)}_t) + \frac{1}{2\eta_t} \|w^{(m)}_t - w_{t-1}\|^2 \le R^{(m)}_{B^{(m)}_t} (w) + \frac{1}{2\eta_t} \|w - w_{t-1}\|^2 - \frac{1/\eta_t - \nu}{2}\|w^{(m)}_t - w\|^2 .
\end{equation}
Let us denote for any $t\ge 1$
\[
\bar w_{t-1}:= \prox_{\rho \bar R}(w_{t-1}) = \argmin_{w} \left\{\bar R(w) + \frac{1}{2\rho}\|w- w_{t-1}\|^2\right\}.
\]
Setting $w=\bar w_{t-1}$ in the right hand side of~\eqref{proof:fedmspp_nonsmooth_key_1} yields
\[
R^{(m)}_{B^{(m)}_t} (w^{(m)}_t) + \frac{1}{2\eta_t} \|w^{(m)}_t - w_{t-1}\|^2 \le R^{(m)}_{B^{(m)}_t} (\bar w_{t-1}) + \frac{1}{2\eta_t} \|\bar w_{t-1} - w_{t-1}\|^2 - \frac{1/\eta_t - \nu}{2}\|w^{(m)}_t - \bar w_{t-1}\|^2 .
\]
In view of the above inequality we can show that for any $\xi\in I_t$,
\begin{equation}\label{proof:fedmspp_nonsmooth_key_2}
\begin{aligned}
&R^{(\xi)}_{B^{(\xi)}_t} (w_{t-1}) + \frac{1}{2\eta_t} \|w^{(\xi)}_t - w_{t-1}\|^2 \\
=& R^{(\xi)}_{B^{(\xi)}_t} (w^{(\xi)}_t) + \frac{1}{2\eta_t} \|w^{(\xi)}_t - w_{t-1}\|^2 + R^{(\xi)}_{B^{(\xi)}_t} (w_{t-1}) -  R^{(\xi)}_{B^{(\xi)}_t} (w^{(\xi)}_t)\\
\le& R^{(\xi)}_{B^{(\xi)}_t} (w^{(\xi)}_t) + \frac{1}{2\eta_t} \|w^{(\xi)}_t - w_{t-1}\|^2 + G\|w_{t-1} - w^{(\xi)}_t\| \\
\le& R^{(\xi)}_{B^{(\xi)}_t} (w^{(\xi)}_t) + \frac{1}{2\eta_t} \|w^{(\xi)}_t - w_{t-1}\|^2 + G^2\eta_t \\
\le& R^{(\xi)}_{B^{(\xi)}_t} (\bar w_{t-1}) + \frac{1}{2\eta_t} \|\bar w_{t-1} - w_{t-1}\|^2 - \frac{1/\eta_t - \nu}{2}\|w^{(\xi)}_t - \bar w_{t-1}\|^2 + G^2\eta_t,
\end{aligned}
\end{equation}
where in the last but one inequality we have applied Lemma~\ref{lemma:inexact_key_2_fedmspp_nonsmooth}. Now recall that $w_t=\frac{1}{I}\sum_{\xi \in I_t} w^{(\xi)}_t$. Then based on triangle inequality we can see that
\[
\begin{aligned}
& \frac{1}{I} \sum_{\xi \in I_t}R^{(\xi)}_{B^{(\xi)}_t} (w_{t-1}) + \frac{1}{2\eta_t} \|w_t - w_{t-1}\|^2 \\
=&\frac{1}{I} \sum_{\xi \in I_t}R^{(\xi)}_{B^{(\xi)}_t} (w_{t-1}) + \frac{1}{2\eta_t} \left\|\frac{1}{I}\sum_{\xi \in I_t} w^{(\xi)}_t - w_{t-1}\right\|^2 \\
\le&  \frac{1}{I} \sum_{\xi \in I_t} \left\{R^{(\xi)}_{B^{(\xi)}_t} (w_{t-1}) + \frac{1}{2\eta_t} \left\| w^{(\xi)}_t - w_{t-1}\right\|^2\right\} \\
\overset{\zeta_1}{\le}& \frac{1}{I} \sum_{\xi \in I_t} \left\{R^{(\xi)}_{B^{(\xi)}_t} (\bar w_{t-1}) + \frac{1}{2\eta_t} \|\bar w_{t-1} - w_{t-1}\|^2 - \frac{1/\eta_t - \nu}{2}\|w^{(\xi)}_t - \bar w_{t-1}\|^2 + G^2\eta_t\right\}\\
\le&  \frac{1}{I} \sum_{\xi \in I_t} R^{(\xi)}_{B^{(\xi)}_t} (\bar w_{t-1}) + \frac{1}{2\eta_t} \|\bar w_{t-1} - w_{t-1}\|^2 - \frac{1/\eta_t - \nu}{2}\left\| \frac{1}{I} \sum_{\xi \in I_t} w^{(\xi)}_t - \bar w_{t-1}\right\|^2 + G^2\eta_t\\
=&\frac{1}{I} \sum_{\xi \in I_t} R^{(\xi)}_{B^{(\xi)}_t} (\bar w_{t-1}) + \frac{1}{2\eta_t} \|\bar w_{t-1} - w_{t-1}\|^2 - \frac{1/\eta_t - \nu}{2}\left\|w_t - \bar w_{t-1}\right\|^2 + G^2\eta_t,
\end{aligned}
\]
where in ``$\zeta_1$'' we have used \eqref{proof:fedmspp_nonsmooth_key_2}. Conditioned on $\mathcal{F}_{t-1}$ , taking expectation (w.r.t. both the randomness of device sampling and data sampling introduced associated with the iteration step $t$) over both sides of the above inequality leads to the following:
\[
\begin{aligned}
& \mathbb{E} \left[\bar R (w_{t-1}) + \frac{1}{2\eta_t} \|w_t - w_{t-1}\|^2  \mid  \mathcal{F}_{t-1}\right]\\
=& \mathbb{E} \left[ \frac{1}{I} \sum_{\xi \in I_t}R^{(\xi)} (w_{t-1}) + \frac{1}{2\eta_t} \|w_t - w_{t-1}\|^2\mid  \mathcal{F}_{t-1}\right]\\
=& \mathbb{E} \left[ \frac{1}{I} \sum_{\xi \in I_t}R^{(\xi)}_{B^{(\xi)}_t} (w_{t-1}) + \frac{1}{2\eta_t} \|w_t - w_{t-1}\|^2\mid  \mathcal{F}_{t-1}\right] \\
\le& \mathbb{E} \left[ \frac{1}{I} \sum_{\xi \in I_t} R^{(\xi)}_{B^{(\xi)}_t} (\bar w_{t-1}) + \frac{1}{2\eta_t} \|\bar w_{t-1} - w_{t-1}\|^2 - \frac{1/\eta_t - \nu}{2}\left\|w_t - \bar w_{t-1}\right\|^2 + G^2\eta_t \mid  \mathcal{F}_{t-1}\right]\\
=& \mathbb{E} \left[ \bar R (\bar w_{t-1}) + \frac{1}{2\eta_t} \|\bar w_{t-1} - w_{t-1}\|^2 - \frac{1/\eta_t - \nu}{2}\left\|w_t - \bar w_{t-1}\right\|^2 + G^2\eta_t \mid  \mathcal{F}_{t-1}\right].
\end{aligned}
\]
Based the above inequality and by applying Lemma~\ref{lemma:inexact_key_2_fedmspp_nonsmooth} again we can show that
\begin{equation}\label{proof:fedmspp_nonsmooth_key_3}
\begin{aligned}
& \mathbb{E} \left[\bar R (w_t) + \frac{1}{2\eta_t} \|w_t - w_{t-1}\|^2  \mid  \mathcal{F}_{t-1}\right]\\
=& \mathbb{E} \left[\bar R(w_{t-1}) + \frac{1}{2\eta_t} \|w_t - w_{t-1}\|^2 + \bar R(w_{t}) - \bar R(w_{t-1}) \mid  \mathcal{F}_{t-1}\right]\\
\le & \mathbb{E} \left[\bar R(w_{t-1}) + \frac{1}{2\eta_t} \|w_t - w_{t-1}\|^2 + G\|w_{t} - w_{t-1}\| \mid  \mathcal{F}_{t-1}\right] \\
\le& \mathbb{E} \left[ \bar R(\bar w_{t-1}) + \frac{1}{2\eta_t} \|\bar w_{t-1} - w_{t-1}\|^2 - \frac{1/\eta_t - \nu}{2}\left\|w_t - \bar w_{t-1}\right\|^2 + 2G^2\eta_t \mid  \mathcal{F}_{t-1}\right],
\end{aligned}
\end{equation}
where in the last inequality we have used $\|w_t - w_{t-1}\|\le \frac{1}{I}\sum_{\xi \in I_t}\|w^{(\xi)}_t - w_{t-1}\| \le G\eta_t$ due to triangle inequality and Lemma~\ref{lemma:inexact_key_2_fedmspp_nonsmooth}.

Since $\bar R$ is also $\nu$-weakly convex, invoking Lemma~\ref{lemma:function_strong_convexity} to $\bar w_{t-1} = \prox_{\rho \bar R_{\erm}}(w_{t-1})$ yields
\[
\bar R(\bar w_{t-1}) + \frac{1}{2\rho} \|\bar w_{t-1} - w_{t-1}\|^2 \le \bar R(w_t) + \frac{1}{2\rho} \|w_t - w_{t-1}\|^2 - \frac{1/\rho - \nu}{2}\|\bar w_{t-1} - w_t\|^2,
\]
which immediately gives the following conditioned expectation bound:
\begin{equation}\label{proof:fedmspp_nonsmooth_key_4}
\begin{aligned}
&\mathbb{E}\left[\bar R(\bar w_{t-1}) + \frac{1}{2\rho} \|\bar w_{t-1} - w_{t-1}\|^2\mid \mathcal{F}_{t-1}\right] \\
\le& \mathbb{E}\left[\bar R(w_t) + \frac{1}{2\rho} \|w_t - w_{t-1}\|^2 - \frac{1/\rho - \nu}{2}\|\bar w_{t-1} - w_t\|^2 \mid \mathcal{F}_{t-1}\right].
\end{aligned}
\end{equation}
By summing up~\eqref{proof:fedmspp_nonsmooth_key_3} and~\eqref{proof:fedmspp_nonsmooth_key_4} we get
\[
\begin{aligned}
&\mathbb{E} \left[ \frac{1/\eta_t - 1/\rho}{2} \|w_t - w_{t-1}\|^2  \mid  \mathcal{F}_{t-1}\right] \\
\le& \mathbb{E} \left[\frac{1/\eta_t - 1/\rho}{2} \|\bar w_{t-1} - w_{t-1}\|^2 -\frac{1/\eta_t + 1/\rho -2\nu}{2} \|\bar w_{t-1} - w_t\|^2 + 2G^2\eta_t\mid  \mathcal{F}_{t-1}\right].
\end{aligned}
\]
Since by assumption $\eta_t\le \rho$, rearranging the terms in the above yields
\[
\begin{aligned}
&\mathbb{E} \left[ \|w_t - \bar w_{t-1}\|^2 \mid \mathcal{F}_{t-1}\right] \\
\le& \frac{1/\eta_t - 1/\rho}{1/\eta_t + 1/\rho - 2\nu}\|\bar w_{t-1} - w_{t-1}\|^2 + \frac{4G^2\eta_t}{1/\eta_t + 1/\rho - 2\nu } \\
\le& \|\bar w_{t-1} - w_{t-1}\|^2 - \frac{2(1/\rho - \nu)}{1/\eta_t + 1/\rho - 2\nu} \|\bar w_{t-1} - w_{t-1}\|^2 + \frac{4G^2\eta_t}{1/\eta_t + 1/\rho - 2\nu }.
\end{aligned}
\]
Then based on the above and the definition of Moreau envelope we can show that
\[
\begin{aligned}
&\mathbb{E} \left[ \bar R_{\rho}(w_t) \mid \mathcal{F}_{t-1}\right] \\
=& \mathbb{E} \left[ \bar R(\bar  w_t) + \frac{1}{2\rho} \|\bar w_t - w_t\|^2\mid \mathcal{F}_{t-1}\right] \\
\le& \mathbb{E} \left[ \bar R(\bar  w_{t-1}) + \frac{1}{2\rho} \|\bar w_{t-1} - w_t\|^2\mid \mathcal{F}_{t-1}\right] \\
=& \bar R(\bar  w_{t-1}) + \frac{1}{2\rho} \mathbb{E} \left[ \|\bar w_{t-1} - w_t\|^2\mid \mathcal{F}_{t-1}\right]\\
\le& \bar R(\bar  w_{t-1}) + \frac{1}{2\rho}\|\bar w_{t-1} - w_{t-1}\|^2 - \frac{(1/\rho - \nu)/\rho}{1/\eta_t + 1/\rho - 2\nu} \|\bar w_{t-1} - w_{t-1}\|^2 + \frac{2G^2\eta_t/\rho}{1/\eta_t + 1/\rho - 2\nu }\\
=& \bar R_{\rho}(w_{t-1}) - \frac{(1/\rho - \nu)/\rho}{1/\eta_t + 1/\rho - 2\nu} \|\bar w_{t-1} - w_{t-1}\|^2 + \frac{2G^2\eta_t/\rho}{1/\eta_t + 1/\rho - 2\nu } \\
=& \bar R_{\rho}(w_{t-1}) - \frac{1 - \rho\nu}{1/\eta_t + 1/\rho - 2\nu} \left\|\nabla \bar R_{\rho}(w_{t-1})\right\|^2 + \frac{2G^2\eta_t/\rho}{1/\eta_t + 1/\rho - 2\nu } ,
\end{aligned}
\]
where in the last equality we have used the identity $\|\bar w_{t-1} - w_{t-1}\|^2 = \rho^2 \left\|\nabla \bar R_{\rho}(w_{t-1})\right\|^2$~\citep[see, e.g., ][]{davis2019stochastic}. By rearranging the terms in the above and taking expectation over $\mathcal{F}_{t-1}$ we obtain that
\[
\frac{1 - \rho\nu}{1/\eta_t + 1/\rho - 2\nu} \mathbb{E}\left[\left\|\nabla \bar R_{\rho}(w_{t-1})\right\|^2\right] \le \mathbb{E}\left[\bar R_{\rho}(w_{t-1})\right] -\mathbb{E}\left[\bar R_{\rho}(w_t)\right]+ \frac{2G^2\eta_t/\rho}{1/\eta_t + 1/\rho - 2\nu }.
\]
Averaging the above over $t=1,...,T$ yields
\[
\begin{aligned}
\frac{1}{T} \sum_{t=0}^{T-1}  \mathbb{E}\left[\left\|\nabla \bar R_{\rho}(w_{t})\right\|^2\right] \le& \frac{1/\eta_t + 1/\rho - 2\nu}{T(1-\rho\nu)} \mathbb{E}\left[\bar R_{\rho}(w_{0}) - \bar R_{\rho}(w_{T})\right] + \frac{2G^2\eta_t}{\rho(1-\rho\nu)} \\
\le& \frac{1/\eta_t + 1/\rho - 2\nu}{T(1-\rho\nu)} \bar \Delta_{\rho} + \frac{2G^2\eta_t}{\rho(1-\rho\nu)} \\
=& \frac{(1 - 2\rho \nu)\bar \Delta_{\rho}}{T\rho (1-\rho\nu)} + \frac{\bar \Delta_{\rho}}{\eta_t T(1-\rho\nu)} + \frac{2G^2\eta_t}{\rho(1-\rho\nu)}\\
\le& \frac{\bar \Delta_{\rho}}{T\rho} + \frac{2\bar \Delta_{\rho}}{\eta_t T} + \frac{4G^2\eta_t}{\rho} \\
=& \frac{\bar \Delta_{\rho}}{T\rho} + \frac{2\bar \Delta_{\rho}+ 4G^2\rho}{\rho \sqrt{T}},
\end{aligned}
\]
where in the last but one inequality we have used $\rho<\frac{1}{2\nu}$, and in the last inequality we have used the choice of $\eta_t\equiv \frac{\rho}{\sqrt{T}}$. The desired bound follows by preserving the dominant terms in the above bound and appealing to the definition of $t^*$.
\end{proof}

\section{Proofs of Preliminary Lemmas}
\label{apdx:proofs_preliminary}
Here we provide the proofs of some auxiliary lemmas introduced in Appendix~\ref{apdx:preliminary}.

\subsection{Proof of Lemma~\ref{lemma:stability_rERM}}
\label{ssect:proof_stability_rERM}

\begin{proof}
Let $w^*_S = \argmin_{w\in \mathbb{R}^p} R^r_S(w)$. Based on the strong convexity of $ R^r_S(w_S) $ we can see that
\[
\frac{\lambda}{2} \|w_S -  w^*_S\|^2 \le R^r_S(w_S) -   R^r_S(w^*_S) \le \varepsilon_t,
\]
which directly implies $\|w_S - w^*_S\| \le \sqrt{\frac{2\varepsilon_t}{\lambda}}$. Let us consider a sample set $S^{(i)}$ which is identical to $S$ except that one of the $z_i$ is replaced by another random sample $z'_i$. Denote $w^*_{S^{(i)}} = \argmin_{w\in \mathbb{R}^p} R^{r}_{S^{(i)}}(w)$. Then we can show that
\[
\begin{aligned}
& R^r_S(w^*_{S^{(i)}}) - R^r_S(w^*_S) \\
=& \frac{1}{N}\sum_{j\neq i} \left(\ell(w^*_{S^{(i)}}; z_j) - \ell(w^*_S;z_j) \right) + \frac{1}{N} \left(\ell(w^*_{S^{(i)}};z_i) - \ell(w^*_S;z_i)\right) + r(w^*_{S^{(i)}}) - r(w^*_S)\\
=& R^r_{S^{(i)}}(w^*_{S^{(i)}}) - R^r_{S^{(i)}}(w^*_S) + \frac{1}{N} \left(\ell(w^*_{S^{(i)}};z_i) - \ell(w^*_S;z_i)\right)  - \frac{1}{N} \left(\ell(w^*_{S^{(i)}};z'_i) - \ell(w^*_S;z'_i)\right) \\
\overset{\zeta_1}{\le}& \frac{1}{N} \left(\ell(w^*_{S^{(i)}};z_i) - \ell(w^*_S;z_i)\right)  - \frac{1}{N} \left(\ell(w^*_{S^{(i)}};z'_i) - \ell(w^*_S;z'_i)\right) \\
\overset{\zeta_2}{\le}& \frac{2G}{N}\|w^*_{S^{(i)}} - w^*_S\|,
\end{aligned}
\]
where ``$\zeta_1$'' follows from the optimality of $w_{S^{(i)}}$ and ``$\zeta_2$'' is due to the Lipschitz continuity of loss. The strong convexity of $R^r_S$ implies
\[
R^r_S(w^*_{S^{(i)}}) - R^r_S(w^*_S) \ge \frac{\lambda}{2} \|w^*_{S^{(i)}} - w^*_S\|^2.
\]
Combining the preceding two inequalities yields
\[
\|w^*_{S^{(i)}} - w^*_S\| \le \frac{4G}{\lambda N}.
\]
Therefore by triangle inequality and the above bounds we get
\[
\begin{aligned}
\|w_{S^{(i)}} - w_S\| =& \|w_{S^{(i)}} -  w^*_{S^{(i)}} +  w^*_{S^{(i)}} - w^*_S + w^*_S - w_S \| \\
\le& \|w_{S^{(i)}} -  w^*_{S^{(i)}}\| +  \|w^*_{S^{(i)}} - w^*_S\| + \|w^*_S - w_S \| \\
\le& \frac{4G}{\lambda N} + 2\sqrt{\frac{2\varepsilon_t}{\lambda}}
\end{aligned}
\]
which implies the desired uniform stability as the above holds for any pair of $S^{(i)}$ and $S$.
\end{proof}

\subsection{Proof of Lemma~\ref{lemma:stability_genalization_first_order_moment}}
\label{ssect:proof_stability_genalization_first_order_moment}

\begin{proof}
Let us consider a sample set $S^{(i)}$ which is identical to $S$ except that one of the $Z_i$ is replaced by another random sample $Z'_i$. Since $S$ and $S^{(i)}$ are both i.i.d. samples of the data distribution. It follows that
\[
 \mathbb{E}_{S}\left[ \nabla R(A(S))\right] = \mathbb{E}_{S^{(i)}}\left[ \nabla R (A(S^{(i)})) \right] = \mathbb{E}_{S^{(i)}\cup \{Z_i\}}\left[\nabla \ell(A(S^{(i)}));Z_i) \right]= \mathbb{E}_{S \cup \{Z'_i\}}\left[\nabla \ell(A(S^{(i)});Z_i) \right].
\]
Since the above holds for all $i=1,...,N$, by averaging the above equality over the training data we obtain  that
\begin{equation}\label{inequat:proof_exp_risk_smooth_key_2}
\mathbb{E}_{S}\left[ \nabla R(A(S))\right] = \frac{1}{N}\sum_{i=1}^N  \mathbb{E}_{S \cup \{Z'_i\}}\left[\nabla \ell(A(S^{(i)});Z_i) \right].
\end{equation}
Regarding the empirical case, by definition we have
\[
\mathbb{E}_S\left[ \nabla R_S(A(S))\right] =  \frac{1}{N}\sum_{i=1}^N \mathbb{E}_S\left[\nabla \ell(A(S);Z_i) \right] = \frac{1}{N}\sum_{i=1}^N \mathbb{E}_{S\cup \{Z'_i\}}\left[ \nabla \ell(A(S);Z_i) \right].
\]
Combining the preceding two equalities gives that
\[
\begin{aligned}
\left\|\mathbb{E}_S \left[\nabla R(A(S)) - \nabla R_S(A(S))\right]\right\| =& \left\| \frac{1}{N}\sum_{i=1}^N  \mathbb{E}_{S \cup \{Z'_i\}}\left[\nabla \ell(A(S^{(i)});Z_i) - \nabla \ell(A(S);Z_i) \right]\right\| \\
\le& L\left\| A(S^{(i)}) - A(S) \right\|\le L \gamma,
\end{aligned}
\]
where we have used the uniform stability of $A$.

To prove the second inequality, again by smoothness of the loss function we have
\[
\left\| \nabla R(A(S)) - \nabla R(A(S^{(i)})) \right\| \le L \left\|A(S) - A(S^{(i)})\right\|\le L\gamma.
\]
Then it follows from Lemma~\ref{lemma:efron_stein} that
\[
\mathbb{E}_S \left[\left\|\nabla R(A(S)) - \mathbb{E}_S\left[\nabla R_S(A(S))\right]\right\|^2\right] \le L^2\gamma^2 N.
\]
The proof is completed.
\end{proof}

\end{document}